\documentclass{lmcs}
\pdfoutput=1
\usepackage[utf8]{inputenc}

\usepackage{lastpage}
\lmcsdoi{20}{3}{17}
\lmcsheading{}{\pageref{LastPage}}{}{}%
{Jul.~25,~2023}{Aug.~27,~2024}{}


\usepackage{todonotes}
\usepackage{comment}
\usepackage{amsmath,amssymb,amsthm,thmtools,thm-restate}
\usepackage[all]{foreign}
\usepackage{multirow}         
\usepackage{bigdelim}         
\usepackage{extarrows}        
\usepackage{braket}           
\usepackage{lipsum}
\usepackage{hyperref}
\usepackage[capitalise]{cleveref}
\usepackage{etoolbox}
\usepackage{booktabs}

\usepackage{tikz}
\usepackage{xcolor}
\usepackage{packages/graphics}
\usepackage{packages/timelines}
\usepackage{packages/commands}
\usepackage{packages/symbols}
\definecolor{primary}{HTML}{87cefa}

\Crefname{thm}{Theorem}{Theorems}
\Crefname{lem}{Lemma}{Lemmas}
\Crefname{defi}{Definition}{Definitions}

\newtheorem{observation}[thm]{Observation}
\Crefname{observation}{Observation}{Observations}

\begin{document}

\title[Controller Synthesis for Timeline-based Games]{Controller Synthesis for Timeline-based Games\rsuper*}
\titlecomment{{\lsuper*}A preliminary version of this article has appeared in the proceedings of GandALF 2022~\cite{Acampora_2022}.}

\author[R.~Acampora]{Renato Acampora}[a]
\author[L.~Geatti]{Luca Geatti\lmcsorcid{0000-0002-7125-787X}}[a]
\author[N.~Gigante]{Nicola Gigante\lmcsorcid{0000-0002-2254-4821}}[b]
\author[A.~Montanari]{Angelo Montanari\lmcsorcid{0000-0002-4322-769X}}[a]
\author[V.~Picotti]{Valentino Picotti\lmcsorcid{0000-0001-7713-1461}}[c]

\address{University of Udine, Italy}
\email{%
  acampora.renato@spes.uniud.it, luca.geatti@uniud.it, 
  angelo.montanari@uniud.it%
}

\address{Free University of Bozen-Bolzano, Italy}
\email{nicola.gigante@unibz.it}

\address{University of Southern Denmark}
\email{picotti@imada.sdu.dk}

\begin{abstract}
In the timeline-based approach to planning, the evolution over time of a set of
state variables (the timelines) is governed by a set of temporal constraints.
Traditional timeline-based planning systems excel at the integration of planning
with execution by handling \emph{temporal uncertainty}. In order to handle
general nondeterminism as well, the concept of \emph{timeline-based games} has
been recently introduced. It has been proved that finding whether a winning
strategy exists for such games is \EXPTIME[2]-complete. However, a concrete
approach to synthesize controllers implementing such strategies is missing. This
article fills the gap by providing an effective and computationally optimal
approach to controller synthesis for timeline-based games.
\end{abstract}

\maketitle


\section{Introduction}
\label{sec:introduction}
Automated planning is the field of \emph{artificial intelligence} that studies the development of autonomous agents able of reasoning about how to reach some goals, starting from a high-level description of their operating environment. It is one of the most studied fields of AI, with early work going several decades back~\cite{McCarthyH69,FikesN71}.
Most of the research by the planning community focuses on the
\emph{action-based} approach, where planning problems are modeled in terms of
\emph{actions} that an agent has to perform to suitably change its \emph{state}.
The task is to devise a sequence of such actions that lead to the goal when
executed starting from a given initial state~\cite{FikesN71,FoxL03}.

In this paper, we focus on the alternative paradigm of \emph{timeline-based planning}, an approach born and developed in the space sector~\cite{Muscettola94}. In timeline-based planning, there is no explicit separation among actions, states, and goals. Planning domains are represented as systems of independent but interacting components, whose behavior over time, the \emph{timelines}, is governed by a set of temporal constraints, called \emph{synchronization rules}.

Over the years, timeline-based planning systems have been developed and successfully exploited by space agencies on both sides of the Atlantic~\cite{CestaCFOP06, CestaCDDFOPRS07, FrankJ03, BernardiniS07, ChienRKSEMESFBST00},  for short- to long-term mission planning~\cite{ChienRTTDNASVGA15} as well as on-board autonomy~\cite{FratiniCORD11}. The main advantage of such a paradigm in these contexts is the ability of these systems of handling both planning and \emph{execution} in a uniform way: by the use of \emph{flexible timelines}, timeline-based planners can produce robust plans that, during execution, can be adapted to the current contingency.

However, flexible timelines currently employed in timeline-based systems only
handle \emph{temporal uncertainty}, where the precise timings of events in the
plan are unknown, but the causal sequence of the events is determined.
In particular, they cannot generate robust plans against an environment empowered with
general \emph{nondeterminism}. To overcome this limitation, the concept of
\emph{timeline-based games} was recently introduced~\cite{GiganteMOCR20}. In
timeline-based games, state variables belong either to the controller or to the
environment. The controller aims at satisfying its set of \emph{system} rules, while
the environment can make arbitrary moves, as long as the \emph{domain} rules 
that define the game arena are satisfied. A controller's strategy is
winning if it guarantees that the controller wins, regardless of the choices
made by the environment. The moves available to the two players can determine
both \emph{what} happens and \emph{when} it happens, thus handling temporal uncertainty
and general nondeterminism in a uniform way.

Determining whether a winning strategy exists for timeline-based games has been
proved to be \EXPTIME[2]-complete~\cite{GiganteMOCR20}. However, there is
currently no effective way to synthesize a controller that implements such
strategies. A necessary condition for synthesizing a finite-state strategy and
the corresponding controller is the availability of a \emph{deterministic}
arena. Two methods to obtain such an arena have been followed in the literature, 
but both have limitations and turn out to be inadequate. On the one hand, 
the complexity result of~\cite{GiganteMOCR20} relies on the construction 
of a (doubly exponential) \emph{concurrent game structure} used to model
check some Alternating-time Temporal Logic formulas~\cite{AlurHK02}. 
Even though such a structure is deterministic and theoretically suitable 
to solve a reachability game and to synthesize a controller, its construction 
relies on theoretical nondeterministic procedures that are not realistically 
implementable. On the other hand, Della~Monica~\etal~\cite{DellaMonicaGMS18} 
devised an automata-theoretic solution that provides a concrete and effective 
way to construct an automaton that accepts a word if and only if the original 
planning problem has a solution plan. Unfortunately, the size of the resulting 
\emph{nondeterministic} automaton is already doubly exponential, and its
determinization would result in a further blowup and thus in a non-optimal procedure.

The present paper fills the gap by developing an effective and computationally
optimal approach to synthesizing controllers for timeline-based games. The
proposed method addresses the limitations of previous techniques by directly
constructing a \emph{deterministic} finite-state automaton of an optimal
doubly-exponential size, that recognizes solution plans. Such an automaton can be
turned into the arena for a reachability game, for which many controller
synthesis techniques are available in the literature.
The paper is a significantly revised and extended version of~\cite{Acampora_2022}. 
It provides a detailed account of the general framework, gives some illustrative 
examples, and fully works out all the proofs. 

The rest of the paper is organized as follows. After discussing related work in 
\cref{sec:related}, \cref{sec:preliminaries} introduces timeline-based planning and games.
\Cref{sec:automaton} presents the main technical contribution of the paper,
namely, the construction of the deterministic automaton that recognizes solution
plans. \Cref{sec:games} shows how to turn such an automaton into the arena of a
suitable game from which the controller can be synthesized. \cref{sec:conclusions} 
summarizes the main contributions of the work and suggest future research directions. 
All the technical proofs are included in the appendix.


\section{Related work}
\label{sec:related}

The paradigm of timeline-based planning has been first introduced to plan and schedule scientific operations of the Hubble space telescope~\cite{Muscettola94}. In the following two decades, many timeline-based planning systems have been developed both at NASA and ESA, including EUROPA~\cite{BedraxWeissGBEI05}, ASPEN~\cite{ChienRKSEMESFBST00}, and APSI~\cite{DonatiPCFOCPSRNS08}. Such systems have been used both for short- to long-term mission planning, \eg for the renowned Rosetta mission~\cite{ChienRTTDNASVGA15}, and for onboard autonomy~\cite{FratiniCORD11}. Elements of the timeline-based and the action-based paradigm have been combined into the Action Notation Modeling Language (ANML)~\cite{SmithFC08}, extensively used at NASA since then.

Despite the real-world success, the timeline-based planning paradigm lacked a
thorough foundational understanding in contrast to the action-based one, which
has been extensively studied from a theoretical perspective from the
start~\cite{McCarthyH69,Bylander94}. To enable theoretical investigations into
timeline-based planning, Cialdea Mayer \etal~\cite{CialdeaMayerOU16} laid down
the core features of the paradigm, describing them in a uniform formalism, which
has  been later studied in several contributions. The formalism was compared to
traditional action-based languages like STRIPS, and it was proved that the
latter are expressible by timeline-based languages~\cite{GiganteMMO16}. The
timeline-based plan existence problem was proved to be
\EXPSPACE-complete~\cite{GiganteMMO17} over discrete time in the general case,
and \PSPACE-complete with qualitative constraints only \cite{DellaMonicaGTM20}. On
dense time, the problem goes from being \NP-complete to undecidable, depending
on the applied syntactic restrictions~\cite{BozzelliMMPW20}. Additionally,
logical~\cite{DellaMonicaGMSS17} and automata-theoretic~\cite{DellaMonicaGMS18}
counterparts have been investigated to study the expressiveness of timeline-based
languages.

The above body of work focuses on \emph{deterministic} timeline-based planning domains. However, the paradigm also fits to \emph{uncertain} domains requiring robust plans. Current timeline-based planning systems employ the concept of \emph{flexible timelines}, described as including uncertainty in the timings of events, representing envelopes of possible executions of the plan. Planners, when possible, produce \emph{strongly controllable} flexible plans, whose execution is then robust for the given temporal uncertainty. In order to obtain controllers for executing strongly controllable flexible plans, the problem can be simplified by reducing it to \emph{timed game automata}~\cite{OrlandiniFCF11}.

While the current approach works fairly well in handling temporal uncertainty, it does not support scenarios where the environment is fully nondeterministic. Furthermore, as pointed out in \cite{GiganteMOCR20}, the language of timeline-based planning as formalized in \cite{CialdeaMayerOU16} allows one to write domains that are not solvable by strongly controllable flexible plans, but that may easily be by strategies coping with general nondeterminism. For this reason, \cite{GiganteMOCR20} introduced the concept of \emph{timeline-based game}, which is the focus of this work. Timeline-based games adopt a game-theoretic point of view, where the controller and the environment play by constructing timelines, with the controller trying to fulfill its synchronization rules independently from the choices of the environment. This setting allows one to handle both temporal uncertainty and general nondeterminism, thus strictly generalizing previous approaches based on flexible timelines. In \cite{GiganteMOCR20}, the problem of deciding the existence of a winning strategy for a given timeline-based game has been proved to be \EXPTIME[2]-complete. The proof is based on the construction of a \emph{concurrent game structure} where a suitable \emph{alternating-time temporal logic}  (ATL) formula is model checked~\cite{AlurHK02}. However, the construction relies on nondeterministic procedures that are not effectively implementable, and thus it does not solve the problem of synthesizing actual controllers for timeline-based games. This work fills the gap by providing 
an effective synthesis algorithm.

The devised algorithm builds on classical results in the field of \emph{reactive synthesis}, which studies how to build correct-by-construction controllers satisfying high-level logical specifications. The original formulation of the problem of reactive synthesis is due to Church~\cite{church1962logic}. The problem  for \emph{S1S} specifications was later solved by Büchi and Landweber using a non-elementary complexity algorithm~\cite{buchi1990solving}. As for Linear Temporal Logic (LTL) specifications, the problem is \EXPTIME[2]-complete \cite{pnueli1989synthesis,rosner1992modular}, which, interestingly,  is the same complexity as timeline-based games. In both cases, the core of the synthesis algorithm is the construction of a \emph{deterministic} arena, where the game can be solved with a fix-point computation. This work focuses on constructing such an arena for timeline-based games (\cref{sec:automaton,sec:games}).

\section{Preliminaries}
\label{sec:preliminaries}

In this section, we provide an overview of the general framework that underpins
our work. We begin by introducing the general features of timeline-based planning, and then we
discuss timeline-based games. Next, we introduce the reactive synthesis problem.
Finally, we recall the concept of \emph{difference bound matrices}
(DBMs)~\cite{dill1989timing,peron2007abstract}, which are the data structures that we will use
to represent the temporal constraints of a system.

\subsection{Timeline-based planning}
The first basic notion is that of \emph{state variable}.
\begin{defi}[State variable]
  \label{def:timelines:state-variable}
  A \emph{state variable} is a tuple $x=(V_x,T_x,D_x,\gamma)$, where:
  \begin{itemize}
  \item $V_x$ is the \emph{finite domain} of $x$;
  \item $T_x:V_x\to2^{V_x}$ is the \emph{value transition function} of $x$,
    which maps each value $v\in V_x$ to the set of values that can immediately
    follow it; 
  \item $D_x:V_x\to\N\times\N$ is the \emph{duration function} of $x$, mapping
    each value $v\in V_x$ to a pair $(\dmin, \dmax)$ specifying respectively the
    minimum and maximum duration of any interval where $x=v$;
  \item $\gamma:V_x\to\set{\mathsf{c},\mathsf{u}}$ is the 
    \emph{controllability tag}, that, for each value $v\in V_x$, specifies whether it is
    \emph{controllable} $\left(\gamma\left(v\right)=\mathsf{c}\right)$ or \emph{uncontrollable}
    $\left(\gamma\left(v\right)=\mathsf{u}\right)$.
  \end{itemize}
\end{defi}

A state variable  $x$ takes its values from a finite domain and represents a
finite state machine with a transition function $T_x$. The behavior over time of
a state variable $x$ is modeled by a timeline. Intuitively, a \emph{timeline}
for a state variable $x$ is a finite sequence of \textit{tokens}, that is,
contiguous time intervals where $x$ holds a given value.

Following the approach described in~\cite{GiganteMOCR20}, instead of formally
defining timelines in terms of tokens, we represent executions of timeline-based
systems as single words, called \emph{event sequences}, where each event
describe the start/end of some token in a given time point.

To this end, we first define the notion of action.

\begin{defi}
  Let $\SV$ be a set of state variables. An \emph{action} is a term of the form $\tokstart(x,v)$ or $\tokend(x,v)$, where $x\in\SV$ and $v\in V_x$.
\end{defi}

Actions of the form $\tokstart(x,v)$ are \emph{starting} actions, and those of the form $\tokend(x,v)$ are \emph{ending} actions. We denote by $\actions_\SV$ the set of all the actions definable over a set of state variables $\SV$.

\begin{defi}[Event sequence~\cite{GiganteMOCR20}]
  \label{def:event-sequence}
  Let $\SV$ be a set of state variables and $\actions_\SV$ be the set of all
  the \emph{actions} $\tokstart(x,v)$ and $\tokend(x,v)$, for $x\in\SV$ 
  and $v\in V_x$. An \emph{event sequence} over $\SV$ is a sequence
  $\evseq=\seq{\event_1,\ldots,\event_n}$ of pairs $\event_i=(A_i,\delta_i)$,
  called \emph{events}, where $A_i\subseteq\actions_\SV$ 
  and $\delta_i\in\N^+$, such that, for any $x\in\SV$:
  \begin{enumerate}
  \item \label{def:event-sequence:start}
        for all $1\le i\le n$, if $\tokstart(x,v)\in A_i$, for some $v\in V_x$,
        then there is no $\tokstart(x,v')$ in any $\event_j$ before the
        closest event $\event_k$, with $k > i$, such that $\tokend(x,v)\in A_k$ (if
        any);
  \item \label{def:event-sequence:end}
        for all $1\le i\le n$, if $\tokend(x,v)\in A_i$, for some $v\in V_x$,
        then there is no $\tokend(x,v')$ in any $\event_j$ after the
        closest event $\event_k$, with $k < i$, such that $\tokstart(x,v)\in A_k$ (if
        any);
  \item \label{def:event-sequence:gaps-right}
        for all $1\le i < n$, if $\tokend(x,v)\in A_i$, for some $v\in V_x$, then
        $\tokstart(x,v')\in A_i$, for some $v'\in V_x$;
  \item \label{def:event-sequence:gaps-left}
        for all $1< i \le n$, if $\tokstart(x,v)\in A_i$, for some $v\in V_x$,
        then $\tokend(x,v')\in A_i$, for some $v'\in V_x$.
  \end{enumerate}
\end{defi}

The first two conditions guarantee correct parenthesis placement by identifying
the start and the end of each token in the sequence. Condition 1 prevents a
token from starting before the end of the previous one, while condition 2
prevents the occurrence of two consecutive ends not interleaved by a start.
Conditions 3 and 4 ensure seamless continuity: each token's end (resp., start)
is consistently followed (preceded) by the start (resp., end) of another, except
for the first (resp., last) event in the sequence. These latter conditions
prevent gaps in the timeline description of the represented plan.

In event sequences, a \texttt{token} for a variable $x$ is a maximal interval
with at most one occurrence of events $\event_i=(A_i,\delta_i)$ and
$\event_j=(A_j,\delta_j)$, where $\tokstart(x,v)\in A_i$ and $\tokend(x,v)\in
A_j$, for some $v\in V_x$. We say such a token \emph{starts} at position $i$ and
\emph{ends} at position $j$. Note that \cref{def:event-sequence} implies that a
token that has started is not required to end before the end of the sequence and
that it can end without the corresponding starting action ever appearing. If
this is the case, we say that an event sequence is \emph{open} either to the right or
to the left. Otherwise, it is said to be \emph{closed}. An event sequence closed
to the left and open to the right is called a \emph{partial plan}. Notice that
the empty event sequence $\epsilon$ is closed on both sides for any variable.
Furthermore, in closed event sequences, the first event contains only start
actions, while the last one contains only end actions, one for each variable
$x$.

Given an event sequence $\evseq=\seq{\event_1,\ldots,\event_n}$ over a set of state variables $\SV$, with $\event_i=(A_i,\delta_i)$, we define $\delta(\evseq)$ as $\sum_{1<i\le n}\delta_i$, that is, $\delta(\evseq)$ is the time elapsed from the start to the end of the event sequence (its duration). For any subsequence $\seq{\event_i, \ldots,\event_j}$ of $\evseq$, abbreviated $\slice\evseq_{i,j}$, we denote by $\delta_{i,j}$ (or, equivalently, $\delta(\slice\evseq_{i,j})$) the amount of time spanning that subsequence. Notice that $\delta_{i,j}$ is defined as $\sum_{i<k\le j}\delta_k$. Finally, given an event sequence $\evseq=\seq{\event_1,\ldots,\event_n}$, we define $\evseq_{<i}$ as $\seq{\event_1,\ldots,\event_{i-1}}$, for each $1<i\le n$.

\smallskip

In timeline-based planning, the objective is to satisfy a set of
\emph{synchronization rules}, that specify the desired behavior of the system
(constraints and goal). These rules relate tokens, possibly belonging to
different timelines, through temporal relations among their endpoints. Let \SV
be a set of state variables and $\toknames = \set{a,b,\ldots}$ be a set of
\emph{token names}. 

\begin{defi}[Atom]
    \label{def:atom}
    An atom is a temporal relation between tokens' endpoints of the form $\production{term} \before_{l,u} \production{term}$, where $l\in\N$, $u\in\N\cup\set{+\infty}$, $l \le u$, and a \emph{term} is either $\tokstart(a)$ or $\tokend(a)$, for some $a\in\toknames$.
\end{defi}

As an example, the atom $\tokstart(a) \before_{3,7} \tokend(b)$ constrains token $a$ to start at least $3$ and at most $7$ time units before the end of token $b$, while the atom $\tokstart(a) \before_{0,+\infty} \tokstart(b)$ simply constrains  token $a$ to start before token $b$.

\begin{defi}[Synchronization rule]
    A synchronization rule \Rule has one of the following two forms:
\begin{gather*}\label{eq:synchronisation-rules}
  \begin{array}{rcl}
  \centering
    \langle rule \rangle &:=& a_0[x_0=v_0]\implies \langle body \rangle\quad\\
    \langle rule \rangle &:=& \top\implies \langle body \rangle\quad \\
    \langle body \rangle &:=& \E_1 \lor \E_2 \lor \dots \lor \E_k\quad \\
    \E_j&\bydef&\exists a_1[x_1=v_1] a_2[x_2=v_2]\ldots a_n[x_n=v_n] \suchdot \clause_j, \ for \ 1 \leq j \leq k,
  \end{array}  
\end{gather*} 
where $a_i \in \toknames$, $x_i \in \SV$,  $v_i \in V_{x_i}$, and $\clause_j$ is a conjunction of atoms, for $0 \leq i\leq n$.
\end{defi} 
Terms $a_i[x_i=v_i]$ are referred to as \emph{quantifiers}. The term $a_0[x_0=v_0]$ is called the \emph{trigger}. The disjuncts in the body are called \emph{existential statements}. Quantifiers refer to tokens with the corresponding variable and value. The intuitive semantics of a synchronization rule can be given as follows: for every token satisfying the trigger, at least one of the existential statements must be satisfied as well. Each existential statement $\E_j$ requires the existence of tokens that satisfy the quantifiers in its prefix and the clause $\clause_j$. A token that satisfies the trigger of a rule is said to \emph{trigger} that rule. The trigger of a rule can be empty ($\top$). In such a case, the rule is referred to as \emph{triggerless} and it requires the satisfaction of its body without any precondition.

Let $a$ and $b$ be token names. Here are two examples of synchronization rules
(relations $=$ and $\before$ are syntactic sugar for $\before_{0,0}$ and
$\before_{0,+\infty}$, respectively):
\begin{align*}
  a[x_s=\mathsf{Comm}] \implies {}
  & \exists b[x_g=\mathsf{Available}] \suchdot
    \tokstart(b) \before \tokstart(a) \land \tokend(a) \before \tokend(b)\\
  a[x_s=\mathsf{Science}] \implies {}
  & \exists b[x_s=\mathsf{Slewing}] \ c[x_s=\mathsf{Earth}] \ d[x_s=\mathsf{Comm}]
    \suchdot {}\\
  & \tokend(a) = \tokstart(b) \land \tokend(b) = \tokstart(c) \land
    \tokend(c) = \tokstart(d)
\end{align*}
where variables $x_s$ and $x_g$ represent the state of a spacecraft and the visibility of the communication ground station, respectively. The first synchronization rule requires the satellite and the ground station to coordinate their communications so that when the satellite is transmitting, the ground station is available for reception. The second one instructs the system to send data to Earth after every measurement session, interleaved by the required slewing operation. Triggerless rule can be used to state the \emph{goal} of the system. As an example, the following rule ensures that the spacecraft performs some scientific measurement:
\begin{equation*}
  \true \implies\exists a[x_s=\mathsf{Science}]
\end{equation*}
Triggerless rules only require the existence of tokens specified by the
existential statements, being their universal quantification trivial. In fact,
they are syntactic sugar, as it is possible to translate them into triggered
rules, as shown in~\cite{GiganteMOCR20}. From now on, we will not consider them
anymore.

We now formalise the above intuitive account of the semantics of synchronization
rules.
\begin{defi}[Matching functions~\cite{Gigante19}]
  \label{def:matching-function}
  Let $\evseq=\seq{\event_1,\ldots,\event_n}$ be a (possibly open) event
  sequence, $\E\equiv \exists a_1[x_1=v_1]\ldots
  a_k[x_k=v_k]\suchdot\clause$ be one of the existential statements of a
  synchronization rule $\Rule\equiv a_0[x_0=v_0]\implies
  \E_1\lor\ldots\lor\E_m$, and $V$ be a set of terms such that $\tokstart(a)\in
  V$ or $\tokend(a) \in V$ only if $a \in \set{a_0,\ldots, a_k}$. A
  \emph{matching function} $\gamma:V\to[1,\ldots,n]$ maps each term $T\in V$
  to an event $\event_{\gamma(T)}$ in $\evseq$, such that:
  \begin{enumerate}
  \item \label{def:matching-function:nodes}
        for each $T\in V$, with $T=\tokstart(a)$ (resp., $T=\tokend(a)$),
        if $a$ is quantified as $a[x=v]$ in $\E$, then the event
        $\event_{\gamma(T)}=(A_T,\delta_T)$ is such that $\tokstart(x,v)\in A_T$
        (resp.,~$\tokend(x,v)\in A_T$);
      \item \label{def:matching-function:tokens} if both
        $T=\tokstart(a)$ and $T'=\tokend(a)$ belong to $V$ for some token name
        $a\in\toknames$, then $\gamma(T)$ and $\gamma(T')$ identify the
        endpoints of the same token.
  \end{enumerate}
\end{defi}
As a matter of fact, in \cite{Gigante19}, matching functions are defined in terms of \emph{rule graphs}, a data structure that we do not use here. For this reason, we reformulated the original definition in terms of event sequences.

The following definition gives a formal account of the semantics of synchronization rules.
\begin{defi}[Semantics of synchronization rules]
  Let $\Rule\equiv a_0[x_0=v_0]\implies \E_1\lor\ldots\lor\E_m$ and let
  $\evseq=\seq{\event_1,\ldots,\event_n}$ be an event sequence. We say that $\Rule$
  is \emph{satisfied} by $\evseq$ if, \emph{for each} event $\event_i=(A_i,
  \delta_i)$ such that $\tokstart(x_0,v_0)\in A_i$, there exist an existential
  statement $\E_j\equiv \exists a_1[x_1=v_1]\ldots a_k[x_k=v_k]\suchdot\clause$
  and a matching function $\gamma$ such that if $T \le_{[l,u]} T'$ appears in
  $\clause$, then $l\le \gamma(T')-\gamma(T)\le u$, for any pair of terms $T$ and $T'$.
\end{defi}

\emph{Timeline-based planning problems} can be defined as follows.
\begin{defi}[Timeline-based planning problem]
  A \emph{timeline-based planning problem} is a pair $P=(\SV,\S)$, where $\SV$ is
  a set of state variables and $\S$ is a set of synchronization rules over
  $\SV$. An event sequence $\evseq$ over $\SV$ is a solution plan for $P$ if all
  the rules in $\S$ are satisfied by $\evseq$.
\end{defi}

\subsection{Timeline-based games.} We are now ready to introduce the notion of \emph{timeline-based game}, that subsumes that of \emph{timeline-based planning with uncertainty} given in~\cite{CialdeaMayerOU16}.

\begin{defi}[Timeline-based game]
  \label{def:games:game}
  A \emph{timeline-based game} is a tuple $G=(\SV_C,\SV_E,\S,$ $\D)$,
  where $\SV_C$ and $\SV_E$ are the sets of \emph{controlled}
  and \emph{external} state variables, respectively, and $\S$ and $\D$ are the sets of
  \emph{system}  and \emph{domain} synchronization rules, respectively, both 
  involving variables from $\SV_C$ and $\SV_E$.
\end{defi}

A partial plan for $G$ is a partial plan over the variables
$\SV_C\cup\SV_E$. Let $\partialplans_G$ be the set of all possible partial plans
for $G$, simply $\partialplans$ when there is no ambiguity.
Since the empty event sequence $\epsilon$ is closed and $\delta(\epsilon)=0$, the
\emph{empty} partial plan $\epsilon$ is a good starting point for the game.
Players incrementally build onto a partial plan, starting from $\epsilon$, by
playing actions that specify which tokens to start and (or) to end, adding an
event that extends the event sequence, or complementing the existing
last one.  

Formally, we partition the set  of all the available actions $\actions_\SV$ into those that are playable by either of the two players.\fitpar

\begin{defi}[Partition of player actions]
  \label{def:games:actions-partition}
  Let $\SV=\SV_C\cup\SV_E$. The set $\actions_\SV$ of available actions over
   $\SV$ is partitioned into the sets $\actions_C$ of \charlie's actions and  
   $\actions_E$ of \eve's actions, which are defined as follows:\fitpar
  \begin{align}
    \actions_C = {} &
      \underbrace{%
        \set{\tokstart(x,v)\suchthat x\in\SV_C,\; v\in V_x}%
      }_{\text{start tokens on \charlie's timelines}}\;\cup\;
      \underbrace{%
        \set{\tokend(x,v)\suchthat x\in\SV,\; v\in V_x,\; \gamma_x(v)=\ctag}%
      }_{\text{end controllable tokens}}\\
    \actions_E = {} &
      \underbrace{%
        \set{\tokstart(x,v)\suchthat x\in\SV_E,\; v\in V_x}%
      }_{\text{start tokens on \eve's timelines}}\;\cup\;
      \underbrace{%
        \set{\tokend(x,v)\suchthat x\in\SV,\; v\in V_x,\; \gamma_x(v)=\utag}%
      }_{\text{end uncontrollable tokens}}
  \end{align}
\end{defi}

Hence, players can start tokens for owned variables and end them for values that
they control. Let $d=\max(L, U)+1$, where $L$ and $U$ are the maximum lower and
(finite) upper bounds appearing in any rule of $G$. Note that, by
\cref{def:games:actions-partition}, we may have $x \in \SV_E$ and $\gamma_x(v)=
c$ for some $v \in V_x$. This means that Charlie may control the duration of a
variable that belongs to Eve. This situation is symmetrical to the more common one where Eve controls the duration of a variable that belongs to Charlie, that is, uncontrollable tokens. As an example, Charlie may decide to start a task, without
being able to foresee how long it will take. Similarly, the environment may
trigger the start of a process, \eg fixing a plant fault, but Charlie may be
able to control, to some extent, how long it will take to end it, \eg we can
decide to fix it today or tomorrow. 

Actions combine into \emph{moves} starting (resp., ending) multiple tokens simultaneously.

\begin{defi}[Move]
  \label{def:games:moves}
  A \emph{move} $\move_C$ for \charlie is a term of the form $\wait(\delta_C)$
  or $\play(A_C)$, where $1 \le \delta_C \le d$ and $\emptyset\ne
  A_C\subseteq\actions_C$ is  either a set of \emph{starting} actions or 
  a set of \emph{ending} actions. A \emph{move} $\move_E$ for \eve is a term of the form $\play(A_E)$ or
  $\play(\delta_E,A_E)$, where $1 \le \delta_E \le d$ and $A_E\subseteq\actions_E$ is
  either a set of \emph{starting} actions or a set of \emph{ending} actions.
\end{defi}
By \cref{def:games:moves}, moves like $\play(A_C)$ and $\play(\delta_E,A_E)$ can
play either $\tokstart(x,v)$ actions only or $\tokend(x,v)$ actions only. A move
of the former kind is called a \emph{starting} move, while a move of the latter
kind is called an \emph{ending} move. We consider $\wait$ moves as \emph{ending}
moves. Starting and ending moves must alternate during the game.

Let us denote the sets of \charlie's and \eve's moves by $\moves_C$ and $\moves_E$ , respectively. A round of the game is defined as follows.

\begin{defi}[Round]
  \label{def:games:round}
  A \emph{round} $\round$ is a pair
  $(\move_C,\move_E)\in\moves_C\times\moves_E$ of moves such that:
  \begin{enumerate}
  \item \label{def:games:round:alternation}
        $\move_C$ and $\move_E$ are either both \emph{starting} or both
        \emph{ending} moves;
  \item \label{def:games:round:paring}
        either $\round=(\play(A_C),\play(A_E))$, or
        $\round=(\wait(\delta_C),\play(\delta_E,A_E))$, with
        $\delta_E\le\delta_C$;
  \end{enumerate}
\end{defi}

A \emph{starting} (resp., \emph{ending}) round is one made of starting (resp.,
ending) moves. Since \charlie cannot play empty moves and $\wait$ moves are
ending moves, each round is unambiguously either a starting or an ending round.
Moreover, since $\play(\delta_E,A_E)$ moves are always paired with
$\wait(\delta_C)$ ones, which are ending moves, then $\play(\delta_E,A_E)$ moves
are necessarily ending moves (item \ref{def:games:round:alternation} of
Definition \ref{def:games:round}).

 We can now specify how to apply a round to the current partial plan to obtain the new one. The game always starts with a single starting round.

\begin{defi}[Outcome of rounds]
  \label{def:games:round-outcome}
  Let $\evseq=\seq{\event_1,\ldots,\event_n}$ be an event sequence, with
  $\event_n=(A_n,\delta_n)$ ($\event_n=(\emptyset,0)$ if $\evseq=\epsilon$). Let
  $\round=(\move_C,\move_E)$ be a round, $A_E$ and $A_C$ be the sets of actions
  of the two moves ($A_C$ is empty if $\move_C$ is a $\wait$ move), and
  $\delta_E$ and $\delta_C$ be the time increments of the moves. We define
  $\delta_C=1$ (resp., $\delta_E=1$) for $\play(A_C)$  
  (resp., $\play(A_E)$).
  
  The \emph{outcome} of the application of $\round$ on $\evseq$ is the event sequence
  $\round(\evseq)$ defined as follows:
  \begin{enumerate}
  \item \label{def:games:round-outcome:starting}
        if $\round$ is a starting round, then $\round(\evseq)=\evseq_{< n}\event_n'$,
        where $\event_n'\nobreak=\nobreak(A_n\cup A_C\cup A_E,\delta_n)$;
  \item \label{def:games:round-outcome:ending}
        if $\round$ is an ending round, then $\round(\evseq)=\evseq\event'$, where
        $\event'=(A_C\cup A_E,\delta_E)$;
  \end{enumerate}
  
  We say that $\round$ is \emph{applicable} to $\evseq$ if:
  \begin{enumerate}[label=\alph*)]
  \item \label{def:games:round-outcome:integrity}
        $\round(\evseq)$ 
        complies with \cref{def:event-sequence};
  \item \label{def:games:round-outcome:alternation}
        $\round$ is an ending round if and only if $\evseq$ is open for all variables that appear in the moves.
  \end{enumerate}
\end{defi}

A single move by either player is applicable to $\evseq$ if there is
a move for the other player such that the resulting round is applicable to
$\evseq$.
The game starts from the empty partial plan $\epsilon$, and players play in
turn, composing a round from the move of each one, which is applied to the
current partial plan to obtain the new one.
We can now define the notion of \emph{strategy} for each player and that of \emph{winning strategy} for \charlie.

\begin{defi}[Strategy]
  \label{def:games:strategies}
  A \emph{strategy for Charlie} is a function
  $\strategy_C:\partialplans\to\moves_C$ that maps any given partial plan
  $\evseq$ into a move $\move_C$ applicable to $\evseq$.
  A \emph{strategy for Eve} is a function
  $\strategy_E:\partialplans\times\moves_C\to\moves_E$ that maps a partial
  plan $\evseq$ and a move $\move_C\in\moves_C$ applicable to $\evseq$ into a move
  $\move_E$ such that the round $\round=(\move_C,\move_E)$ is applicable to
  $\evseq$.
\end{defi}

A sequence $\rounds=\seq{\round_0,\ldots,\round_n}$ of rounds is called a
\emph{play} of the game. A play is said to be \emph{played according to} some
strategy $\strategy_C$ for \charlie, if, starting from the initial partial plan
$\evseq_0=\epsilon$, it holds that $\round_i=(\strategy_C(\Pi_{i-1}),
\move_E^i)$, for some $\move_E^i$, for all $0<i\le n$, and to be played
according to some strategy $\strategy_E$ for \eve if $\round_i=(\move_C^i,
\strategy_E(\Pi_{i-1},\move_C^i))$, for all $0<i\le n$. It can be easily seen that for
any pair of strategies $(\strategy_C,\strategy_E)$ and any $n\ge0$, there is a
unique play $\rounds_n(\strategy_C,\strategy_E)$ of length $n$ played according to both $\strategy_C$ and $\strategy_E$.

Then, we say that a partial plan $\evseq$ and the play $\rounds$ such that
$\evseq=\rounds(\epsilon)$ are \emph{admissible}, if the partial plan satisfies
the domain rules, and that they are \emph{successful} if the partial plan satisfies the
system rules.

\begin{defi}[Admissible strategy for \eve]
  \label{def:games:admissible-strategy}
  A strategy $\strategy_E$ for \eve is \emph{admissible} if for each strategy
  $\strategy_C$ for \charlie, there is $k\ge 0$ such that the play
  $\rounds_k(\strategy_C,\strategy_E)$ is admissible.
\end{defi}

\charlie wins if, \emph{assuming} that domain rules are respected, he manages to
satisfy the system rules no matter how \eve plays.

\begin{defi}[Winning strategy for \charlie]
  \label{def:games:winning-strategy}
  Let $\strategy_C$ be a strategy for \charlie. We say that $\strategy_C$ is a
  \emph{winning strategy} for \charlie if for any \emph{admissible} strategy
  $\strategy_E$ for \eve, there exists $n\ge0$ such that the play
  $\rounds_n(\strategy_C,\strategy_E)$ is successful.
\end{defi}

We say that \charlie \emph{wins} the game $G$ if he has a winning strategy,
while \eve \emph{wins} the game if a winning strategy for \charlie does not
exist.

\subsection{Synthesis}
The synthesis problem is the problem of devising an implementation that satisfies a formal specification of an input-output relation~\cite{PnueliRosner89}. Such an implementation may be a transducer, a Mealy machine, a Moore machine, a circuit, or the like. In the following, we give a short account of the roles of games and strategies in game-based synthesis.

\begin{defi}[Game Graph] A finite game graph $G$ is a triple $\left(Q, Q_C, E\right)$, where $Q$ is a finite set of nodes, $Q_C \subseteq Q$ is the subset of \charlie's nodes, and $E \subseteq Q \times Q$ is a transition relation. The relation $E$ must satisfy the condition: $\forall q \exists q' : \left(q,q'\right) \in E$ (totality). \end{defi} 

A \emph{play} on a game graph $G$ starting from the initial state $q_0$ is an infinite sequence $p = q_0 q_1 q_2\ldots$, where $(q_i, q_{i+1}) \in E$, for all  $i \ge 0$. A game is  a pair $(G, \mathcal{W})$, where $G$ is a game graph and $\mathcal{W}$ 
is the winning condition of the game. In the general case, $\mathcal{W}$ consists of the set of plays won by  \charlie.

Here, we focus on reachability winning conditions, which are expressed as $\mathcal{W} \bydef \{ R \subseteq Q \mid R \cap F \neq \emptyset \}$, for a given set $F \subseteq Q$. A play $p$ is said to satisfy 
$\mathcal{W}$ if the set of states visited by $p$, denoted by $occ(p) = \{q \in Q \mid \exists i \mathrel{.} p(i) = q\}$, intersects $\mathcal{W}$, that is, \charlie wins the play $p$ if $p$ visits at least one state in $F$.

\begin{defi}[Reachability game] A reachability game is a pair $(G, \mathcal{W})$, where $G = (Q, Q_C, E)$ is a game graph and $\mathcal{W}$ is a reachability winning condition. \end{defi}

A strategy for \charlie is a function $f : Q^*\cdot Q_C \rightarrow Q$. A play
$p$ adheres to strategy $f$ if, for each $q_i \in Q_C$, $q_{i+1} = f(q_0 \ldots
q_i)$. Given an initial state $q$, a strategy for \charlie is a winning strategy
if \charlie wins any play from $q$ that follows the strategy $f$. The same holds
for \eve. \charlie (resp., \eve) wins if a winning strategy exists from $q$. 

Given a game $(G, \mathcal{W})$, with $G = (Q, Q_C, E)$,
the winning region of \charlie is defined as $W_C \bydef \{ q \in Q \, \vert \,
\charlie \text{ wins from q } \}$. The winning region $W_E$ for \eve
is defined in an analogous way.
The two sets are clearly disjoint ($W_C \cap W_E = \emptyset$). The game is
said to be \emph{determined} if $W_C \cup W_E = Q$. It is well known that
reachability games are determined~\cite{Thomas2008}. 

\smallskip

The next step is to build a Controller starting from a winning strategy $f$ such that the specification is met. We use Moore machines as \charlie plays first. 

\begin{defi}[Moore machine]
    \label{def:moore-machine}
    A Moore machine is a tuple $M = (Q, \Sigma, \Gamma, q_0, \delta, \tau)$, where $Q$ is a finite set of states, $\Sigma$ is a finite input alphabet, $\Gamma$ is a finite output alphabet, $q_0 \in Q$ is the initial state, $\delta : Q \times \Sigma \rightarrow Q$ is the transition function, and $\tau: Q \rightarrow \Gamma$ it the output function.
\end{defi}

By suitably tying $\delta$ and $\tau$ to $f$, one can effectively implement $f$. We refer the reader to \cref{def:controller-implementation} for the details on how we do it.

\subsection{Difference Bound Matrices}
\emph{Difference bound matrices} (DBMs) were introduced by
Dill~\cite{dill1989timing} as a pragmatic representation of constraints $(x - y
\leq c)$. Later on, Péron et al.~\cite{peron2007abstract} suitably expanded
the formalism. The following short account of the formalism is basically borrowed from
the latter work,

Let $\mathit{Var} = \{v_0, v_1, \ldots v_n\}$ be a finite set of variables, $\bar{V} = \mathbb{Z} \cup
\{+\infty\}$ be a set of values that variables and constants can take, and $C$ be  a set of constraints of the form $v_i - v_j \leq c$, where $v_i,v_j \in \mathit{Var}$ and $c \in \bar{V}$. 
The DBM that represents $C$ is an $(n+1)\times(n+1)$ matrix defined as follows: 
\begin{equation*}
    M_{ij} = \text{inf}\{c \, \mid \, (v_i - v_j \leq c) \in C\},
\end{equation*}
where $\text{inf}(\emptyset) = +\infty$. 

$M_{ij}$ equals the tightest value of
$c$ if there is some constraint $(v_i - v_j \leq c)$ in $C$; otherwise, it is
$+\infty$. The variable $v_0 \in \mathit{Var}$ is always valued to $0$, and it is
used to express bounds on variables, that is, $v_i \leq c$ is written as
$v_i - v_0 \leq c$.  In \cref{sec:automaton}, we use DBMs to conveniently
represent atoms (see \cref{def:atom}).


\section{A deterministic automaton for timeline-based planning}
\label{sec:automaton}

In this section, we define an encoding of timeline-based planning problems into \emph{deterministic} finite state automata (DFA).
Given a timeline-based planning problem, the corresponding automaton recognizes all and only those \emph{event sequences} that represent solution plans for the problem. In the next section, we will use such an automaton as the game arena for a timeline-based game.

\subsection{Plans as words}
Let $P=(\SV, S)$ be a timeline-based planning problem and, as already stated in the previous section, let $d = \max(L, U)+1$, where $L$ and $U$ are the maximum lower and (finite) upper bounds appearing in any rule of $P$. We restrict our attention to event sequences where the distance between two consecutive events is at most $d$. Such a restriction guarantees us the finiteness of the considered alphabet, and it does not cause any loss in generality, as proved by Lemma 4.8 of~\cite{Gigante19}. Moreover, it agrees with the notion of move of a timeline-based game (see \cref{def:games:moves}).

We define the symbols of the alphabet $\Sigma$ as \emph{events} of the form $\event = \pair{A,
\delta}$, where $A \subseteq \actions_\SV$ and $1\le\delta\le d$. Formally,
$\Sigma=2^{\actions_\SV}\times\ar{d}$, where $\ar{d}=\set{1,\ldots,d}$. Note that the size of $\Sigma$ is exponential in the size of the problem. Moreover,
we define $\window(P)$ as the sum of all the
coefficients appearing as upper bounds in the rules of $P$. This value represents the
maximum amount of time a rule can ``count'' far away from the
occurrence of the quantified
tokens. Consider, for instance, the following rule: 
\begin{align}
  \label{eq:example-3var-synch}
  a_0[x_0=v_0]\to{} &\exists a_1[x_1=v_1] a_2[x_2=v_2] a_3[x_3=v_3] \suchdot \\
  &\tokstart(a_1)\before_{4,14}\tokend(a_0)
  \land \tokend(a_0)\before_{0,+\infty}\tokend(a_2) \land \tokstart(a_2)\before_{0,3}\tokend(a_3)\tag*{}
\end{align}

In this case, assuming the above rule  to be the only one in the problem,
$\window(P)$ would be $3 + 14=17$. Thus, the rule can account for what happens
at most $17$ time points from the occurrence of its quantified tokens. For
instance, if the token $a_1$ appears at a specific distance from $a_0$, it has
to be within less than $17$ time points, and any modification of the plan that
alters this distance can break the rule's satisfaction. However, what occurs
further away from $a_0$ only affects the fulfillment of the rule
\emph{qualitatively}. Suppose that the tokens $a_2$ and $a_3$ are, together, at
$100$ time points from $a_0$. Changing this distance while maintaining the
qualitative order between tokens does not break the rule's satisfaction. For
$\window(P)$'s properties refer to \cite{Gigante19}.

\subsection{Matching structures}
A key insight underlying the construction we are going to outline is that every atomic temporal
relation $T \before_{l,u} T^\prime$ can be rewritten as the conjunction of two upper 
bound constraints $T^\prime - T \leq u$ and $T - T^\prime \leq -l$, where we
represent a lower bound constraint $T^\prime - T \geq l$ as an upper bound one.
This allows us to rewrite the clause \clause of an existential statement \E as a constraint system
$\nu(\clause)$ with constraints of the form $T - T^\prime \leq n$, for $n \in
\Z \cup \{+\infty\}$.

The  constraint system $\nu(\clause)$ can be represented by a difference bound matrix $D$ indexed by terms, where the entry $D[T, T']$ gives the upper bound $n$ on $T - T^\prime$.
In building $D$, we ensure the right duration of tokens
by augmenting the system with constraints
of the kind $\tokstart(a_i)-\tokend(a_i)\leq-\dmin^{x_i=v_i}$ and
$\tokend(a_i)-\tokstart(a_i)\leq\dmax^{x_i=v_i}$, for any quantified token
$a_i[x_i=v_i]$ of $\E$.
As an example, the constraint system and the DBM for the above rule are the ones in \cref{fig:rule-constraint-system,fig:dbm}, respectively.

\begin{figure}[h]
    \centering
    \begin{align*}
    \begin{cases}
        \tokend(a_0) - \tokstart(a_1) &\leq 14 \\
        \tokstart(a_1) - \tokend(a_0) &\leq -4 \\
        \tokend(a_0) - \tokend(a_2) &\leq 0 \\
        \tokend(a_3) - \tokstart(a_2) &\leq 3 \\
        \tokstart(a_2) - \tokend(a_3) &\leq 0
    \end{cases}
\end{align*}
    \caption{The constraint system of \cref{eq:example-3var-synch}.}
    \label{fig:rule-constraint-system}
\end{figure}

\begin{figure}
  \begin{equation*}
    \begin{array}{rcccccccc}\toprule  
      &\!\tokstart(a_0) & \!\tokend(a_0) & \!\tokstart(a_1) & \!\tokend(a_1) &
      \!\tokstart(a_2) & \!\tokend(a_2) & \!\tokstart(a_3) & \!\tokend(a_3) \\\midrule
      \tokstart(a_0) \\
      \tokend(a_0) & & & 14 & & & 0 \\
      \tokstart(a_1) & & -4 \\
      \tokend(a_1) \\
      \tokstart(a_2)  & & & & & & & & 0 \\
      \tokend(a_2) & & & \\
      \tokstart(a_3) \\
      \tokend(a_3)  & & & & & 3 & & & \\\bottomrule
    \end{array}
  \end{equation*}
  \caption{DBM of \cref{eq:example-3var-synch}. Missing entries are intended to be $+\infty$.}
  \label{fig:dbm}
\end{figure}

On top of DBMs, we define the concept of \emph{matching structure}, a data
structure that allows us to monitor and update the fulfillment of atomic
temporal relations among terms throughout the execution of the plan. More
precisely, it allows us to manipulate and reason about existential statements of
which only a portion of the requests has been satisfied by the word read so far,
while the rest is potentially satisfiable in the future.

\begin{defi}[Matching Structure]
  \label{def:matching-structure}
  Let $\E\equiv \exists a_1[x_1 = v_1] \dots a_m[x_m = v_m] \,.\, \clause$ be
  an existential statement of a synchronization rule $\Rule \equiv a_0[x_0 =
  v_0] \rightarrow \E_1 \lor \dots \lor \E_k$ over the set of state variables
  \SV. The \emph{matching structure} for $\E$ is a tuple $\M_{\E} = (V, D, M, t)$,
  where:
  \begin{itemize}
  \item $V$ is the set of terms $\tokstart(a)$ and $\tokend(a)$, for
    $a\in\set{a_0, \dots, a_m}$;
  \item 
    $D$ is a DBM of size $|V| \times |V|$, indexed by terms of $V$, whose
    entries take value over $\Z \cup \{+\infty\}$, where
  \begin{align*}
      \begin{cases}
          D[T,T']=n &\quad if  \; T-T'\le n \ \in \nu(\clause), \\
          D[T,T']=0 &\quad if \; T=T', \\
          D[T,T']=+\infty &\quad otherwise;
      \end{cases}
  \end{align*}    
  \item $M \subseteq V$ and $0\le t \le \window(P)$.
  \end{itemize}
\end{defi}

The set $M$ contains the set of terms from $V$ correctly seen in the sequence so
far. We say these terms have been \emph{matched} by the matching structure. We
use $\overline{M} = V \setminus M$ to refer to terms yet to be matched. We say a
matching structure $\M$ to be \emph{closed} if $M = V$, \emph{initial} if $M =
\emptyset$, and \emph{active} if $\tokstart(a_0) \in M$ and it is not closed.
The component $t$ represents the time elapsed since matching $\tokstart(a_0)$.
As time progresses, we update a matching structure as follows.

In the DBMs of a matching structure, the bounds between any pair of terms $T$
and $T'$, with one in $M$ while the other not, are tightened by the elapsing of
time. When $T\in M$ and $T'\in\overline{M}$, $D[T,T']$ is a lower bound loosened
by adding the elapsed time $\delta$. When $T\in\overline{M}$ and $T'\in M$,
$D[T,T']$ is an upper bound tightened by subtracting $\delta$. Consider the DBM
in \Cref{fig:dbm} and the pair of terms $\tokstart(a_1)$ and $\tokend(a_0)$. We
have $D[\tokstart(a_1),\tokend(a_0)]=-4$, implying that
$\tokstart(a_1)-\tokend(a_0)\le -4$ must hold. Suppose that $\tokstart(a_1)\in
M$ (it has been matched), and that $\tokend(a_0)\in\overline{M}$ (it needs to be
matched). Now, in a time step, the entry in the DBM is incremented and updated
to $-4+1=-3$ reflecting the fact that we now have $3$ time steps left to match
$\tokend(a_0)$. A similar analysis leads us to the conclusion that the entry
$D[\tokend(a_0),\tokstart(a_1)]=14$ has to be decremented by $1$ and updated to
$14-1=13$. This intuition is formalized as follows.

\begin{defi}[Time shifting]
  \label{def:time-shift}
  Let $\delta > 0$ be a positive amount of time, and let $\M = (V, D, M, t)$ be a
  matching structure. The result of shifting $\M$ by $\delta$ time units,
  written $\M + \delta$, is a matching structure $\M^\prime = (V, D^\prime, M,
  t')$, where:
  \begin{itemize}
  \item for all $T, T' \in V$:
    \begin{equation*}
      D^\prime[T,T'] =
      \begin{cases}
        D[T,T'] + \delta &\text{if } T \in M \text{ and } T' \in
        \overline{M}\\
        D[T,T'] - \delta &\text{if } T \in \overline{M} \text{ and } T' \in
        M\\
        D[T,T'] &\text{otherwise}
      \end{cases}
    \end{equation*}
  \item and
    \[
      t' =
      \begin{cases}
        t+\delta & \text{if } \M \text{ is \emph{active}}\\
        t & \text{otherwise}
      \end{cases}
    \]
  \end{itemize}
\end{defi}
\Cref{def:time-shift} specifies how to update the entries of $D$ and how to update $t$ to the trigger occurrence of an active matching structure.

\begin{defi}[Matching]
  \label{def:matching}
  Let $\M = (V, D, M, t)$ be a matching structure and $I \subseteq \overline{M}$
  a set of matched terms. A matching structure $\M^\prime = (V, D, M^\prime, t)$
  is the result of matching the set $I$, written $\M \cup I$, with $M^\prime = M
  \cup I$.
\end{defi}

To correctly match an existential statement while reading an event sequence, a matching structure is updated only as long as one witnesses no violation of temporal constraints. As such, we deem an event as \emph{admissible} or not.

\begin{defi}[Admissible Event]\label{def:admissible-event}
  An event $\event = (A, \delta)$ is \emph{admissible} for a matching structure
  $\M_{\E} = (V, D, M, t)$ if and only if, for every $T \in M$
  and $T' \in \overline{M}$, $\delta \leq D[T',T]$, \ie the elapsing of $\delta$
  time units does not exceed the upper bound of some term $T'$ not yet
matched by $\M_{\E}$.
\end{defi}

Each admissible event $\event$ that is read can be matched with a subset of terms from the matching structure. However, there can be multiple ways to match events and terms. To make this choice explicit, we introduce the following definition.

\begin{defi}[$I$-match Event]\label{def:match-event}
  Let $\M_{\E} = (V, D, M, t)$ be a matching structure and  $I \subseteq
  \overline{M}$. An $I$\emph{-match event} is an admissible event $\event = (A,
  \delta)$ for $\M_{\E}$ such that:
  \begin{enumerate}
  \item for all token names $a \in \mathsf{N}$ quantified as $a[x = v]$ in $\E$
    we have that:\label{def:match-event:good-match}
    \begin{enumerate}
    \item if $\tokstart(a) \in I$, then $\tokstart(x, v) \in A$;
      \label{def:match-event:good-match:start}
    \item $\tokend(a) \in I$ if and only if $\tokstart(a) \in M$ and $\tokend(x,v) \in
      A$;\label{def:match-event:good-match:end}
    \end{enumerate}
  \item and for all $T \in I$ it holds that:\label{def:match-event:relations}
    \begin{enumerate}
    \item \label{def:match-event:preceding-terms} for every other term $T' \in
      V$, if $D[T',T] \leq 0$, then $T' \in M \cup I$;
    \item \label{def:match-event:lower-bounds} for all $T' \in M$, $\delta \geq
      -D[T',T]$, \ie all the lower bounds on $T$ are satisfied;
    \item \label{def:match-event:zero-no-bounds} for each other term $T' \in I$,
      either $D[T',T] = 0$, $D[T,T'] = 0$, or $D[T',T] = D[T, T'] = +\infty$.
    \end{enumerate}
  \end{enumerate}
\end{defi}

We consider an event $\event$ an $I$-match event if its actions correspond to the terms in $I$. The definition in \Cref{def:match-event:good-match} ensures the correct matching of each term to an action it represents and that the endpoints of a quantified token precisely identify the endpoints of a token in the event sequence. Meanwhile, \Cref{def:match-event:relations} guarantees that matching the terms in $I$ does not violate any atomic temporal relation. In addition, \Cref{def:match-event:preceding-terms} deals with the qualitative aspect of a ``happens before'' relation, while \Cref{def:match-event:lower-bounds,def:match-event:zero-no-bounds} address the quantitative aspects of the lower bounds of these relations. It is worth noting that an $\emptyset$-event is also considered admissible.

Let $\matchstructs_P$ denote the set of all matching structures for a planning
problem $P$, and let $\I$ be the set of all possible terms built from token
names in $\toknames$. To describe the evolution of a matching structure, we
define a quaternary relation
$S\subseteq\matchstructs_P\times\Sigma\times\I\times\matchstructs_P$ as
$(\M,\event,I,\M')\in{S}$, for an event $\event = (A, \delta)$, if and only if
$\event$ is an $I$-match event for $\M$, and $\M'=(\M+\delta)\cup I$. We also
write $\M \stepm \M'$ in place of $(\M,\event,I,\M')\in{S}$.
Note that, from \Cref{def:match-event}, a single event can represent multiple
$I$-match events for a matching structure. Therefore, given a matching structure
$\M$ and an event $\event$, automaton states will collect all the matching
structures $\M'$ resulting from the relation $S$, for some set of terms $I$.
Given a set of matching structures $\Upsilon$, this notion is best described by
the function $\step_\event(\Upsilon)=\set{\M' \mid (\M,\event,I,\M')\in S,
  \text{ for some } \M\in\Upsilon \text{ and } I \in \I}$. Furthermore, we
define $\Upsilon^\Rule_t\subseteq\Upsilon$ as the set of all the \emph{active}
matching structures $\M\in\Upsilon$ with timestamp $t$, associated with any
existential statement of $\Rule$. Matching structures in $\Upsilon^\Rule_t$
contribute to fulfilling the same triggering event of $\Rule$, regardless of their
existential statement. We also define $\Upsilon_\bot\subseteq\Upsilon$ as the
set of \emph{non-active} matching structures of $\Upsilon$. Lastly, we say that
$\Upsilon$ is \emph{closed} if there exists $\M\in\Upsilon$ such that $\M$ is
\emph{closed}.

\begin{figure}
  \centering
  \begin{tikzpicture}[xscale=1.6]
    \footnotesize
    \begin{timelines}[flexibility/.style={color=primary!50}]
      \begin{timeline}[var={x_0}]

        \token[length=160] {$x_0 = v_0$};

        \token[length=40] {$x_0 = v^\prime_0$};

      \end{timeline}

      \begin{timeline}[var={x_1}]
        \token[length=60] {$x_1 = v^\prime_1$};

        \token[length=80] {$x_1 = v_1$};

        \token[length=60] {$x_1 = v^{''}_1$};
      \end{timeline}

      \begin{timeline}[var={x_2}]
        \token[length=50] {$x_2 = v'_2$};

        \token[length=110] {$x_2 = v_2$};

        \token[length=40] {$x_2 = v^{''}_2$};
      \end{timeline}

      \begin{timeline}[var={x_3}]
          \token[length=80]{$x_3 = v_3$};

          \token[length=120]{$x_3 = v'_3$};
      \end{timeline}
    \end{timelines}

  \end{tikzpicture}
  \caption{%
    Example of timelines for variables $x_0, \, x_1, \, x_2, \, x_3$.%
  }%
  \label{fig:timelines-0-1-2-3}
\end{figure}

We conclude this section by providing an example of updating a matching structure $\M=(V,D,M,t)$ for the rule discussed at the beginning of the section. Consider the set of timelines in \cref{fig:timelines-0-1-2-3}. Before matching any term $\M$ is initial with $M = \emptyset$, $t = 0$, $D$ as the DBM in \cref{fig:dbm}, and $V$ as the set of term $\tokstart(a)$ and $\tokend(a)$ for $a \in \{a_0, a_1, a_2, a_3\}$.
We begin by matching the terms $\tokstart(a_0)$ and $\tokstart(a_3)$ from the event $\event = (\{\tokstart(x_0, v_0),\tokstart(x_3, v_3)\}, 0)$ (we do not consider $\tokstart(x_1,v'_1)$ and $\tokstart(x_2,v'_2)$ since they are not in $V$). Such event is an $I$-match event for $I = \{ \tokstart(a_0), \tokstart(a_3)\}$: it is an admissible event (\cref{def:admissible-event}), \cref{def:match-event:good-match:start} holds, for both 
$\tokstart(a_0)$ and $\tokstart(a_3)$, there are no terms that should appear before them (\cref{def:match-event:preceding-terms}), there are no related lower bounds (\cref{def:match-event:lower-bounds}), and $D[\tokstart(a_0), \tokstart(a_3)] = D[\tokstart(a_3), \tokstart(a_0)] = +\infty$ (\cref{def:match-event:zero-no-bounds}). Hence, we update $M = M \cup I = \{\tokstart(a_0), \tokstart(a_3)\}$ and $t = t + \delta = 0$; now $\M$ is active. The next term to consider is $\tokstart(a_2)$, which occurs after $\delta = 5$ time steps.

First, we ensure that the event $\event = (\tokstart(x_2, v_2), 5)$ is admissible. We show that by examining the DBM in \cref{fig:dbm}, we see that the elapsing of time $\delta$ does not exceed any upper bound related to terms $T \in M$ and $T' \in \overline{M}$. Next, the set $I$ in the current state appears as $I = \{ \tokstart(a_2) \}$. Notice that we are in the case of \cref{def:match-event:good-match:start}, and \cref{def:match-event:relations} holds because no constraint involves the term $\tokstart(a_2)$ (\cref{def:match-event:preceding-terms}), no lower bounds are related to $\tokstart(a_2)$ (\cref{def:match-event:lower-bounds}), and $\tokstart(a_2)$ is the only term in $I$ (\cref{def:match-event:zero-no-bounds}). Therefore, from \cref{def:matching,def:time-shift}, we update $\M$ as follows: $\M = (\M + \delta) \cup I$. Each entry of the DBM will remain unchanged since the third update case of \cref{def:time-shift} applies, $M = M \cup I = \{\tokstart(a_0), \tokstart(a_3), \tokstart(a_2)\}$, and $t = t + \delta = 5$.

Similarly, for the next event is $\event = (\tokstart(x_1, v_1), 1)$, we check if such an event is admissible, and indeed it is since the upper bound  $D[\tokend(a_0), \tokstart(a_1)] = 9 \ge \delta$. It is also an $I$-match event for $I = \{\tokstart(a_1)\}$, since it respects \cref{def:match-event:good-match:start} and all the relations in \cref{def:match-event:relations}; thus we update $\M$. We decrement $D[\tokend(a_3), \tokstart(a_2)]$ and increment $D[\tokstart(a_2), \tokend(a_3)]$ by 1 (see \cref{def:time-shift}), update $M$ like follows $M = \{\tokstart(a_0), \tokstart(a_3), \tokstart(a_2), \tokstart(a_1)\}$, and $t = t + \delta = 5 + 1 = 6$.

The next event is $\event = (\tokend(x_3, v_3), 3)$ after 2 time steps. Note that it is an admissible event and also an $I$-match event for $I = \{\tokend(a_3)\}$. In this case, we emphasize that \cref{def:match-event:good-match:end,def:match-event:relations} are respected. We update the DBM as follows: $D[\tokend(a_0), \tokstart(a_1)] = 14 - 2 = 12$, $D[\tokstart(a_1), \tokend(a_0)] = -4 + 2 = -2$, $D[\tokend(a_3), \tokstart(a_2)] = 2 - 2 = 0$, $D[\tokstart(a_2), \tokend(a_3)] = 1 + 2 = 3$. Then, we update $M = M \cup I = \{\tokstart(a_0), \tokstart(a_3),\linebreak \tokstart(a_2), \tokstart(a_1), \tokend(a_3)\}$ and $t = t + \delta = 6 + 2 = 8$. Notice that if we did not match $\tokend(a_3)$ now, at the next time step, the timeline would have violated the rule above because the upper bound $D[\tokend(a_3),\tokstart(a_2)] = 0$.

The subsequent event is $\event = (\tokend(x_1, v_1), \tokstart(x_1 = v'_1), 6)$ for which $I = \tokend(a_1)$. Since there is no constraint involving $\tokend(a_1)$, this event is admissible and an $I$-match event. The DBM is shifted by 6 time steps, and $M = \{\tokstart(a_0), \tokstart(a_3), \tokstart(a_2), \tokstart(a_1), \tokend(a_1)\}$.

The last event $\event = (\{\tokstart(x_0, v'_0), \tokstart(x_2, v^{''}_2), \tokend(x_0, v_0), \tokend(x_2, v_2)\}, 2)$ is admissible and an $I$-match for $I = \{\tokend(a_0), \tokend(a_2)\}$, note that there is not an upper bound between $\tokend(a_0)$ and $\tokend(a_2)$ and that \cref{def:match-event:good-match:end,def:match-event:relations} of the definition of $I$-match event are respected.

\subsection{Building the automaton}\label{sec:automata-construction}
We can now define the automaton. First, given an existential statement $\E$, let $\mathbb{E}_\E$ be the set of all existential statements in the same rule of $\E$. Next, let $\mathbb{F}_P$ be the set of functions that map each existential statement of $P$ to a set of existential statements and let $\mathbb{D}_P$ be the set of functions that map each existential statement to a set of matching structures $\Upsilon$.
An automaton $\TV_P$ that checks the transition functions of the variables is easy to define. Then, given a timeline-based planning problem $P=(\SV, S)$, we can characterize the corresponding automaton as $A_P=\TV_P\cap\S_P$. Here, $\S_P$ checks the fulfillment of the synchronization rules, and we define it as $\S_P = (Q, \Sigma, q_0, F, \tau)$ where
\begin{enumerate} 
\item $Q = 2^{\matchstructs_P} \times \mathbb{D}_P \times \F_P \cup \set{\bot}$ is the
  finite set of states, \ie states are tuples of the form $\langle \Upsilon,
  \Delta, \Phi \rangle\in2^{\matchstructs_P} \times \mathbb{D}_P \times \F_P$, plus a
  sink state $\bot$;
\item $\Sigma$ is the input alphabet defined above;
\item the initial state $q_0 = \langle \Upsilon_0, \Delta_0, \Phi_0 \rangle$ is
  such that $\Upsilon_0$ is the set of initial matching structures of the
  existential statements of $P$ and, for all existential statements $\E$ of $P$,
  we have $\Delta_0(\E) = \emptyset$ and $\Phi_0(\E) = \mathbb{E}_\E$;
\item $F \subseteq Q$ is the set of final states defined as:
  \[
    F = \Set{ \langle \Upsilon, \Delta, \Phi \rangle \in Q |
      \begin{gathered}
        \M \text{ is not \emph{active} for all } \M \in
        \Upsilon\\
        \text{and }\Delta(\E)=\emptyset\text{ for all }\E\text{ of } P
      \end{gathered}}
  \]
\item $\tau : Q \times \Sigma \rightarrow Q$ is the transition function that
  given a state $q=\langle \Upsilon, \Delta, \Phi \rangle$ and a symbol $\event
  = (A, \delta)$ computes the new state $\tau(q,\event)$. Let
  $\step^\E_\event(\Upsilon^\Rule_t)=\set{\M_\E \mid
  \M_\E\in\step_\event(\Upsilon^\Rule_t)}$. Moreover, let $\Psi^\Rule_t = \set{ \E |
  \M_{\E} \in \step_\event(\Upsilon^\Rule_t)}$. Then, the updated components of
  the state are based on what follows, where $W = \window(P)$:
  \begin{align*}
    \Upsilon' &= \step_\event(\Upsilon_\bot) \cup \bigcup \Set{
      \step_\event(\Upsilon^\Rule_t) |
      \text{$t\le W-\delta$ and 
      $\step_\event(\Upsilon^\Rule_t)$ is not \emph{closed}}} \\
    \Delta'(\E) &=\begin{cases}
        \step^\E_\event(\Upsilon^\Rule_t) & \text{where $t$ is the minimum such that $t> W-\delta$ and $\step^\E_\event(\Upsilon^\Rule_t)\ne\emptyset$} \\
        \step_\event(\Delta(\E)) & \text{if such $t$ does not exist}
      \end{cases}\\
    \Phi'(\E) &= \begin{cases}
      \mathbb{E}_\E\quad\text{if $\E\in\Psi(\E')$ for some $\E'$ such that 
      $\Delta'(\E')$ is \emph{closed}}  \\
      \Phi(\E) \setminus 
        \set{
          \E'\mid \exists t> W-\delta \suchdot \E'\in\Psi^\Rule_t 
          \land \E\not\in\Psi^\Rule_t
        } \quad \text{otherwise}
    \end{cases}
  \end{align*}

  Let $\Delta''(\E)=\Delta'(\E)$ unless there is an $\E'$ with $\E\in\Phi'(\E')$
  such that $\Delta'(\E')$ is \emph{closed}, in which case
  $\Delta''(\E)=\emptyset$. Then, $\tau(q,\event)=\seq{\Upsilon', \Delta'',
  \Phi'}$ if the following holds:
  \begin{enumerate}
  \item for every $\Upsilon^\Rule_t$, $\step_\event(\Upsilon^\Rule_t) \neq
    \emptyset$, and \label{dfa:delta:no-failed-step}
  \item for every synchronization rule $\Rule \equiv a_0[x_0=v_0] \rightarrow
    \E_1 \lor \dots \lor \E_n$ in $S$, if $\tokstart(x_0, v_0) \in A$, then
    there exists $\M_{\E_i} = (V,D,M,0) \in \Upsilon'$, 
    with $i \in \{1\dots n\}$, such that $\tokstart(a_0) \in M$;\label{dfa:delta:trigger-capture}
  \end{enumerate}
  Otherwise, $\tau(q,\event)=\bot$.
\end{enumerate}

The first component $\Upsilon$ of an automaton's state $q$ is a set of matching structures that keeps track of the occurred events in the last $\window(P)$ time points. The timestamp $t$ of any matching structure in $\Upsilon$ satisfies $t<\window(P)$. These matching structures evolve using the $\step_\event$ function until they become closed or their timestamp reaches $\window(P)$. 

Matching structures that reach $\window(P)$ get promoted to a new role where they record the pieces of existential statements not yet matched to satisfy all the trigger events of $\Rule$ that occurred before the last $\window(P)$ time points. However, the automaton does not store these matching structures in $\Upsilon$. Instead, it uses the function $\Delta$ mapping each existential statement $\E$ of a rule $\Rule$ to the set of matching structures for $\E$ with $t=\window(P)$. Thus, effectively summarizing events happening before this window to keep size under control.

When a set $\Upsilon^\Rule_t$ exceeds the bound $\window(P)$, the $\Delta$ function needs to be updated by merging the information from $\Upsilon^\Rule_t$ with the information already stored in $\Delta$. However, closing a set $\Delta(\E)$ does not necessarily mean that every event that triggered $\Rule$ satisfies $\Rule$. This is because there may be other sets, say $\Delta(\E')$, responsible for fulfilling the same rule $\Rule$, but for different trigger events. Therefore, closing $\Delta(\E)$ alone does not imply that $\Rule$ has been satisfied. Conversely, there may be cases where $\Delta(\E)$ and $\Delta(\E')$ contribute to match the same trigger events, and closing either set is enough to satisfy $\Rule$.

To address the issue of lost information when adding a set of matching structures to $\Delta$, we introduce the $\Phi$ function, mapping existential statements to sets of existential statements, as the third component of the automaton states. For an existential statement $\E$ and for every existential statement $\E' \in \Phi(\E)$, it holds that the set of matching structures $\Delta(\E')$ tracks the satisfaction of the same trigger events as the set $\Delta(\E)$. This way, when a set $\Delta(\E)$ is closed, we can discard its matching structures as well as the matching structures in $\Delta(\E')$.

In \cref{sec:soundness:completeness} we state and prove soundness and
completeness of the automaton construction. Now, instead, let us address the
size of the automaton.

Let us recall that we assumed that the timestamp of each event in an event sequence is bounded. However, it is worth noting that since events may have an empty set of actions, \cref{thm:soundness-completeness} can handle arbitrary event sequences as well, provided that we add suitable empty events. Let us now analyze the size of the automaton. 

\begin{thm}[Size of the automaton]
  Let $P=(\SV, S)$ be a timeline-based planning problem and let $\A_P$ be the
  associated automaton. Then, the size of $A_P$ is at most doubly-exponential in
  the size of $P$.
\end{thm}

\begin{proof}
We define $E$ as the overall number of existential statements in $P$, which is linear in the size of $P$. We can then observe that $\abs{\mathbb{D}_P} \in \O({(2^{\abs{\matchstructs_P}})}^E)= \O(2^{E\cdot\abs{\matchstructs_P}})$, thus the number of $\Delta$ functions is doubly exponential in the size of $P$.
Next, note that $\lvert\mathbb{F}_P\rvert \in \mathcal{O}({(2^E)}^E) = \mathcal{O}(2^{E^2})$. Then, $\abs{\S_P} \in \O(\abs{\Sigma}\cdot 2^{\abs{\matchstructs_P}})$ indicating that the size of $\S_P$ is at most exponential in the number of possible matching structures.
To bound this number, we define $N$ as the largest finite constant appearing in $P$ in any atom or value duration and $L$ as the length of the longest existential prefix of an existential statement occurring inside a rule of $P$. Note that $N$ is exponential in the size of $P$ since constants are expressed in binary, while $L \in \O(\abs{P})$.
We can then observe that the entries of a DBM for $P$, of which the number is quadratic in $L$, are constrained to take values within the interval $\ar{-N, N}$ (excluding the value $+\infty$), which size is linear in $N$. By \Cref{def:matching-structure}, it follows that $\abs{\matchstructs_P} \in \O(N^{L^2} \cdot 2^L \cdot \window(P))$ indicating that the number of matching structures is at most exponential in the size of~$P$.
\end{proof}
Note that our automaton is the same size as the automaton built by Della Monica et al. in \cite{DellaMonicaGMS18}. However, while their automaton is nondeterministic, ours is deterministic: an essential property to achieve the \EXPTIME[2] optimal asymptotic complexity for the synthesis procedure.

\subsection{Soundness and Completeness}
\label{sec:soundness:completeness}
In the following, we present auxiliary notation, definitions, and essential lemmas for establishing the soundness and completeness of the automaton construction. For readability, we have included proofs in the appendix.

\begin{defi}[Run of a matching structure]
  Let $\evseq=\seq{\event_1,\ldots,\event_n}$ be a (possibly open) event
  sequence, and let $\M_\E$ be the initial matching structure of an existential
  statement $\E$. A \emph{run} of $\M_\E$ on $\evseq$ yielding a matching
  structure $\M_n$ is a sequence $\matchseq = \seq{\match_1, \ldots, \match_n}$
  of $I$-match events for the matching structures $\seq{\M_\E, \M_1, \ldots,
    \M_{n-1}}$, such that for every $i \in [1,\ldots,n]$, $\M_{i-1} \stepm[i]
  \M_i$. We write $\M_\E \runm \M_n$ when such run exists, or $\M_\E
  \xlongrightarrow{\evseq} \M_n$, if $\matchseq$ is not relevant.
\end{defi}

To link matching structures with the semantics of synchronization rules we
establish a connection between matching functions (\cref{def:matching-function})
and runs.

\begin{restatable}[Correspondence between runs and matching functions]{lem}{runfuncmap}
  \label{lemma:function-matching-run}
  Let $\evseq=\seq{\event_1,\ldots,\event_n}$ be a (possibly open) event
  sequence, and let $\M_\E$ be the initial matching structure of an existential
  statement $\E\equiv \exists a_1[x_1=v_1]\ldots a_k[x_k=v_k]\suchdot\clause$,
  with $\clause$ augmented with atoms $\tokstart(a_i) \before_{\dmin^{x_i=v_i},
    \dmax^{x_i=v_i}} \tokend(a_i)$, for every $0\leq i \leq k$. Then, there
  exists a run $\matchseq=\seq{\match_1, \ldots, \match_n}$ of $\M_\E$ on
  $\evseq$, yielding a matching structure $\M_n = \tuple{V, D_n, M_n, t_n}$, if
  and only if there exists a matching function $\gamma:M_n \to[1,\ldots,n]$ such
  that, for every atom of the form $T\before_{l,u} T'$ in $\clause$:
  \begin{enumerate}[label=(\Roman*)]
  \item \label{lemma:function-matching-run:entire-atom} if $T' \in M_n$, then also $T
    \in M_n$, $\gamma(T) \le \gamma(T')$, and $l \le
    \delta(\slice\evseq_{\gamma(T),\gamma(T')}) \le u$;
  \item \label{lemma:function-matching-run:partial-atom} if $T' \not\in M_n$, but $T \in
    M_n$, then $\delta(\slice\evseq_{\gamma(T),n}) \le u$.
  \end{enumerate}
  Furthermore, $\gamma$ and $\matchseq$ are such that for every $T \in M_n$, $T
  \in I_{\gamma(T)}$, \ie, they agree on the matching of the terms of $\M_n$. We
  write $M_\E \runm* M_n$, if $\gamma$ corresponds to a run of $\M_\E$, or
  $\evseq,\gamma\models \M_n$, if $\M_\E$ is clear from the context.
\end{restatable}

\begin{observation}
    \label{obs:matching-functions}
    Note that the existence of the matching function $\gamma$ stated by 
    \cref{lemma:function-matching-run}, if the corresponding matching structure is 
    closed, implies the satisfaction of the given existential statement, and 
    \viceversa.
\end{observation}

We now state the core technical result of the completeness proof, which ensures
no important details are lost when matching structures are discarded.

\begin{restatable}{lem}{superset}\label{lemma:matching-structure-superset}
  Let $\evseq =\langle\event_1,\dots,\event_n\rangle$ be an event sequence , let
  $\M_\E$ be the initial matching structure of some existential statement $\E$
  of a rule $\Rule$, and let $\M_r$ be an active matching structure resulting
  from a run $\M_\E \runm*[r] \M_r$, such that $\gamma_r(\tokstart(a_0)) = r$.
  If there exists a run $\M_\E \runm*[s] \M_s$, such that
  $\gamma_s(\tokstart(a_0)) < r$, then there exists a run $\M_\E \runm* \M$,
  such that $\gamma(\tokstart(a_0)) = \gamma_s(\tokstart(a_0))$ and $\M$ matches
  at least as many tokens as $\M_r$.
\end{restatable}

The last needed notion is that of \emph{residual} matching structure, which is
an active matching structure with only infinite bounds.

\begin{defi}[Residual matching structure]\label{def:residual-matching-structure}
  A matching structure $\M = (V, D, M, t)$ is \emph{residual} if it is
  \emph{active} and for every $T \in M$ and $T' \in \overline{M}$, $D[T',T] = +\infty$.
\end{defi}

In other words, $\M$ does not impose any finite upper bound on the distance at
which terms yet to be matched may appear relative to those already matched. The
definition implies that for any residual matching structure, denoted as $\hat\M
= (V, D, M, t)$, every event $\event = (A, \delta)$ is admissible. Additionally,
it is never the case that $\tokstart(a) \in M$ and $\tokend(a) \in \overline{M}$
for any quantified token $a[x = v]$ of $\E$, given that such terms always have a
finite upper bound in $D$ that is at least as strict as the value $\dmax^{x=v}$.
As a result, the ``if'' direction of \Cref{def:match-event:good-match:end} in
the \Cref{def:match-event} of $I$-match never applies to $\hat\M$ for any event
$\event$. Therefore, every event is a valid $\emptyset$-match event for
$\hat\M$.

\begin{observation}\label{obs:residual-run}
  Let $\M_\E \xlongrightarrow{\evseq_1,\matchseq_1}\hat\M$ be a run of the
  \emph{initial} matching structure $\M_\E$, on an event sequence $\evseq_1$,
  yielding a \emph{residual} matching structure $\hat\M$. Then, for any event
  sequence $\evseq_2$, there exists a run $\M_\E
  \xlongrightarrow{\evseq_1\evseq_2,\matchseq_1\matchseq_2} \hat\M'$ such that
  every $I$-match event in $\matchseq_2$ is an $\emptyset$-match event and
  $\hat\M'$ differs from $\hat\M$ by at most the value of the component $t$.
\end{observation}

Consequently, whenever a residual matching structure appears in a run, it has
the potential to remain there indefinitely, which is why it is called
\emph{residual}.

\begin{restatable}[Existence of residual matching structure]{lem}{residualexist}
  \label{lemma:residual-matching-structure}
  Let $\evseq = \seq{\event_1, \ldots, \event_n}$ be an event sequence, and let
  $\M_n$ be an \emph{active} matching structure such that $\evseq, \gamma
  \models \M_n$ and
  $\delta(\slice\evseq_{\gamma(\tokstart(a_0)),n})>\window(P)$. If we consider
  the intermediate matching structures $\seq{\M_1, \ldots, \M_{n-1}}$ of the run
  $\M_\E \runm* \M_n$, then there exists a position $\gamma(\tokstart(a_0)) \leq
  k < n$ such that $\M_k$ is a \emph{residual} matching structure.
\end{restatable}

We are now ready to prove the final result.

\begin{restatable}[Soundness and completeness]{thm}{soundnessCompleteness}
  \label{thm:soundness-completeness}
  Let $P=(\SV, S)$ be a timeline-based planning problem and let $\A_P$ be the
  associated automaton. Then, any event sequence $\evseq$ is a solution plan for
  $P$ if and only if $\evseq$ is accepted by $\A_P$.
\end{restatable}



\section{Controller synthesis}
\label{sec:games}

We leverage the deterministic automaton constructed in the previous section to establish a deterministic arena that enables us to solve a reachability game and determine whether a controller exists. If a controller exists, our procedure allows its synthesis.

\subsection{From the automaton to the arena}

Let $G=(\SV_C, \SV_E, \S, \D)$ be a timeline-based game. The automaton construction we used considered a planning problem with a single set of synchronization rules, while in $G$, we have to account for the roles of both $\S$ and $\D$.

To address this, we define $A_\S$ and $A_\D$ as the deterministic automata constructed over the timeline-based planning problems $P_\S=(\SV_C\cup\SV_E, \S)$ and $P_\D=(\SV_C\cup\SV_E, \D)$, respectively. We then construct the automaton $A_G$ by taking the union of $A_\S$ with the complement of $A_\D$ ($\overline{A_\D}$). Note that these are standard automata-theoretic operations over DFAs.
An accepting run of $A_G$ represents either a plan that violates the domain rules or a plan that adheres to domain and system rules, according to the definition of winning strategy in \cref{def:games:winning-strategy}.
Furthermore, $A_G$ is deterministic, and its size only polynomially increases when built from $A_\D$ and $A_\S$. 

The $A_G$ automaton is not immediately applicable as a game arena since its transitions' labels only reflect events, not game moves. In $A_G$, a single transition can correspond to various combinations of rounds due to the absence of $\wait(\delta)$ moves in the transition's label. For example, an event $\event = (A, 5)$ can arise from either a $\wait(5)$ move by \charlie, followed by a $\play(5, A)$ move by $\eve$, or any $\wait(\delta)$ move with $\delta > 5$ followed by a $\play(5, A)$ move. To obtain a suitable game arena, we need to modify $A_G$ further. 

Let $A_G = (Q,\Sigma, q_0, F, \tau)$ be the automaton constructed as described above. Formally, we define a new automaton $A_G' = (Q,\Sigma,q_0, F,\tau')$ where $\tau'$ is a partial transition function, meaning that the automaton is now incomplete. The function $\tau'$ agrees with $\tau$ on all transitions except those of the form $\tau(q,(\actions,\delta))$, where $\delta>1$ and $\actions$ contains a $\tokend(x,v)$ action with $x\in\SV_C$. In such cases, the transition is undefined in $A_G'$.
An example is shown in Figure \ref{fig:constructions} (left). Note that this removal does not alter the set of plans accepted by the automaton since for each transition $\tau(q,(\actions,\delta))=q'$ with $\delta > 1$, there exist two transitions $\tau(q,(\emptyset,\delta-1))=q''$ and $\tau(q'',(\actions,1))=q'$ in $A_G'$.

\begin{figure}
  \begin{tikzpicture}[state/.style={fill, circle, minimum width=5pt}]
    \path (0,0) node[state] (n1) { }
          (4,4) node[state] (n2) { }
          (4,0) node[state] (n3) { };

    \path[draw,dashed,->] (n1) -- (n2) node[midway, above, sloped, font=\scriptsize] { 
      $\event=(\set{\tokend(x,v)}, 5)$
    };

    \path[draw, ->] (n1) -- (n3) node[midway, above, font=\scriptsize] { 
      $\event'=(\emptyset, 4)$
    };
    \path[draw, ->] (n3) -- (n2) node[midway, above, sloped, rotate=180, font=\scriptsize] { 
      $\event''=(\set{\tokend(x,v)}, 1)$
    };

    \begin{scope}[xshift=5.4cm] 
      \path (0,0) node[state] (n1) { } node[above, outer sep=5pt] {$q$}
            (9,4) node[state] (n2) { } node[above, outer sep=5pt] {$w$};

      \path[draw,dashed,->] (n1) -- (n2) node[midway, above, sloped, font=\scriptsize] { 
        $\event=(\set{\tokend(x,v_1),\tokend(y,w_1),\tokstart(x,v_2),\tokstart(y,w_2)}, 5)$
      };

      \path (3,0) node[state] (q6) { } node[above right] { }
            (3,0.75) node[state] (q5) { } node[above right] { }
            (3,-0.75) node[state] (q7) { } node[above right] { }
            (3,-2.25) node[state] (q10) { } node[above right] { };

      \path[dashed] (q7) -- (q10);

      \path[draw,->] (n1) -- (q5) node[above=-2pt, sloped, near end, font=\scriptsize] { $\wait(5)$ }; 
      \path[draw,->] (n1) -- (q6) node[above, sloped, near end, font=\scriptsize] { $\wait(6)$ };
      \path[draw,->] (n1) -- (q7) node[above, sloped, near end, font=\scriptsize] { $\wait(7)$ };
      \path[draw,->] (n1) -- (q10) node[above, sloped, near end, font=\scriptsize] { $\wait(10)$ };

      \path (6,0) node[state] (q'') { };
      \path[draw,->] (q5) -- (q'') node[above, sloped, midway, font=\scriptsize] { $\play\left(5, 
        \left\{\begin{array}{@{}c@{}}
          \tokend(x,v_1)\\
          \tokend(y,w_1)
        \end{array}\right\}\right)$ };
      \path[draw,->] (q6) -- (q'');
      \path[draw,->] (q7) -- (q'');
      \path[draw,->] (q10) -- (q'') node[below, sloped, midway, font=\scriptsize] { $\play\left(5, 
      \left\{\begin{array}{@{}c@{}}
        \tokend(x,v_1)\\
        \tokend(y,w_1)
      \end{array}\right\}\right)$ };

      \path (9,0) node[state] (q''') { };
      \path[draw,->] (q'') -- (q''') node[midway, above,font=\scriptsize] {
        $\play(\set{\tokstart(x,v_2)})$
      };
      \path[draw,->] (q''') -- (n2) node[midway, sloped,rotate=180, above,font=\scriptsize] {
        $\play(\set{\tokstart(y,w_2)})$
      };
    \end{scope}
  \end{tikzpicture}
  \caption{On the left, the removal of transitions $\event=(A,\delta)$
  with $\delta>1$ and ending actions of controllable tokens in $A$. On the right, the
  transformation of a transition of $A_G$ into a sequence of transitions in
  $A^a_G$, with $x\in\SV_C$, $y\in\SV_E$, and $\gamma_x(v_1)=\gamma_y(w_1)=\mathsf{u}$.}
  \label{fig:constructions}
\end{figure}

To make the game rounds and moves explicit, we can transform the automaton by splitting each transition into four transitions representing the four moves of the two rounds. Starting from the incomplete automaton $A_G'=(Q,\Sigma, q_0, F, \tau')$, we define a new automaton $A_G^a=(Q^a,\Sigma^a, q_0^a, F^a, \tau^a)$ as the game arena.

\begin{enumerate}
  \item The set of states $Q^a$ is given by $Q^a=Q\cup\set{q_\delta\mid 1\le\delta\le d}\cup\set{q_{\delta,A}\mid 1\le\delta\le d, A\subseteq \mathsf{A}}$.
  \item The alphabet $\Sigma^a$ is defined as $\Sigma^a=\moves_C\cup\moves_E$, which corresponds to the set of moves of the two players.
  \item The initial and final states of $A_G^a$ are $q_0^a=q_0$ and $F^a=F$, respectively.
  \item The partial transition function $\tau^a$ is defined as follows. Let $w=\tau(q,\event)$ with $\event=(\actions,\delta)$. We distinguish the cases where $\delta=1$ or $\delta>1$.
    \begin{enumerate}
      \item if $\delta=1$, let $\actions_C\subseteq\actions$ and
      $\actions_E\subseteq\actions$ be the set of actions in $\actions$ playable
      by \charlie and by \eve, respectively. Then:
      \begin{enumerate}
        \item $\tau(q,\play(\actions_C^e))=q_{1,\actions_C^e}$, where 
          $\actions_C^e$ is the set of \emph{ending} actions in $\actions_C$;
        \item $\tau(q_{1,\actions_C^e},\play(\actions_E^e))=q_{1,\actions_C^e\cup\actions_E^e}$, where 
          $\actions_E^e$ is the set of \emph{ending} actions in $\actions_E$;
        \item $\tau(q_{1,\actions_C^e\cup\actions_E^e},\play(\actions_C^s))=q_{1,\actions_C^e\cup\actions_E^e\cup\actions_C^s}$, where 
          $\actions_C^s$ is the set of \emph{starting} actions in $\actions_C$;
        \item $\tau(q_{1,\actions_C^e\cup\actions_E^e\cup\actions_C^s},\play(\actions_E^s))=w$, where 
          $\actions_E^s$ is the set of \emph{starting} actions in $\actions_E$;
      \end{enumerate}
      Here, the states mentioned are added to $Q^a$ as needed.
      \item if $\delta>1$, let $\actions_C\subseteq\actions$ and
      $\actions_E\subseteq\actions$ be the set of actions in $\actions$ playable
      by \charlie and by \eve, respectively. Note that by construction,
      $\actions_C$ only contains \emph{starting} actions. Then:
      \begin{enumerate}
        \item $\tau(q,\wait(\delta_C))=q_{\delta_C}$ for all 
          $\delta\le\delta_C\le d$;
        \item $\tau(q_{\delta_C},\play(\delta, \actions_E^e))=q_{\delta,\actions_E^e}$
          where $\actions_E^e$ is the set of \emph{ending} actions in
          $\actions_E$;
        \item $\tau(q_{\delta,\actions_E^e},\play(\actions_C))=q_{\delta,\actions_E^e\cup\actions_C}$;
        \item $\tau(q_{\delta,\actions_E^e\cup\actions_C},\play(\actions_E^s))=w$ where
          $\actions_E^s$ is the set of \emph{starting} actions in $\actions_E$;
      \end{enumerate}
      where the mentioned states are added to $Q^a$ as needed.
    \end{enumerate}
    All the transitions not explicitly defined above are undefined.
\end{enumerate}

We present a graphical illustration of the above construction in \cref{fig:constructions}. It is worth noting that the automaton preserves the structure of the original automaton $A_G$. For any state, $q\in Q$ and event $\event=(A,\delta)$, any sequence of moves that would result in appending $\event$ to the partial plan (see \cref{def:games:round-outcome}) reaches the same state $w$ in $A^a_G$ as it does in $A_G$ by reading $\event$. Therefore, we can consider $A^a_G$ as being able to read event sequences, even though its alphabet is different. We use the notation $[\evseq]$ to represent the state $q\in Q^a$ reached by reading $\evseq$ in $A^a_G$.
Furthermore, note that, with a slight abuse of notation, any play $\bar\rho$ in the game $G$ is a readable word by the automaton $A_G^a$. Thus, we can establish the following result.
\begin{thm}
  \label{thm:arena-soundness}
  If $G$ is a timeline-based game, for any play $\bar\rho$ for $G$, $\bar\rho$
  is successful if and only if it is accepted by $A_G^a$.
\end{thm}

\subsection{Computing the Winning Strategy and Building the Controller}
Let us define $Q^a_C \subset Q^a$ as the set of states in which \charlie can make a move, and $Q^a_E = Q^a \setminus Q^a_C$ as the set of states where \eve can make a move. Additionally, we define $E=\{(q, q') \in Q^a\times Q^a \mid \exists \event \mathrel{.} \tau^a(q, \event) = q'\}$ as the set of edges in $A^a_G$. By solving the reachability game $(G_R, \mathcal{W})$, where $G_R = (Q^a, Q^a_C, E^a)$ and $\mathcal{W} = \{R \subseteq Q^a \mid R \cap F^a \neq \emptyset \}$, we aim to determine the winning region $W_C$ and the winning strategy $s_C$ for \charlie, provided they exist. In the following discussion, we will show that the winning strategy $\sigma_C$ for the timeline-based game $G$ is derivable from strategy $s_C$ when $q^a_0 \in W_C$.

To determine the winning region $W_C$, we use the well-known \emph{attractor} construction. We are interested to the attractor set of $F^a$ for \charlie, written $Attr_C(F^a)$, thus given $i \ge 0$ we compute the set of states from which \charlie can reach a state $q \in F^a$ within $i$ moves, defined as $Attr_C^i(F^a)$:
\begin{align*} 
  Attr^0_C (F^a) ={}& F^a \\
  Attr^{i+1}_C (F^a) ={}& Attr^i_C (F^a) \\
  &\cup \set{ q^a \in Q^a_C \, | \, \exists r \big((q^a, r) \in E \land r \in Attr^i_C (F^a)\big) } \\
  &\cup \set{ q^a \in Q^a_E \, | \, \forall r \big((q^a, r) \in E \implies r \in Attr^i_C (F^a) \big) }.
\end{align*}

The sequence $Attr^0_C (F^a) \subseteq Attr^1_C (F^a) \subseteq Attr^2_C (F^a) \subseteq \ldots$ eventually becomes stationary for some index $k \leq \lvert Q^a \rvert$, hence we can define $Attr_C (F^a) = \bigcup^{\lvert Q^a \rvert}_{i=0} Attr^i_C(F^a)$ as the attractor set. Note that $W_C = Attr_C(F^a)$ is a known fact for which proof is available in \cite{Thomas2008}.
Next, we want that $q_0^a\in W_C$ since we are interested in a winning strategy $\sigma_C$ for the timeline-based game $G$. If it is the case, by defining $s_C(q) = \mu$ for any $\mu$ such that $\tau^a(q,\mu)=q'$ with $q,q'\in W_C$, which is guaranteed to exist by the attractor construction, we can define $\sigma_C$ for \charlie in $G$ as $\sigma_C(\evseq) = s_C([\evseq])$ for any event sequence $\evseq$. We prove this claim in the following:

\begin{thm}
    \label{thm:winning-region-soundness-completeness}
  Given $A_G^a=(Q^a,\Sigma^a, q_0^a, F^a, \tau^a)$, $q_0^a\in W_C$ if and only
  if $\sigma_C$ is a winning strategy for \charlie for $G$.
\end{thm}
\begin{proof}

  \proofif From the definition of a winning strategy for \charlie in $G$ (\cref{def:games:winning-strategy}), we know that for every admissible strategy $\sigma_E$ for \eve, there exists $n \ge 0$ such that the play $\rounds_n(\strategy_C,\strategy_E)$ is successful. By the soundness of the arena construction (\cref{thm:arena-soundness}), we know that the event sequence $\evseq_n$ representing $\rounds_n(\strategy_C,\strategy_E)$, when seen as a word over $\Sigma^a$, is accepted by $A_G^a$. Therefore, $\evseq_n$ reaches a state in the set $F^a$ starting from $q_0^a$. By the definition of the reachability game, this means that $q_0^a\in W_C$. Thus, we have proved that if $\sigma_C$ is a winning strategy \charlie in $G$, then $q_0^a\in W_C$.
  
  \proofonlyif If $q_0^a\in W_C$, then by definition, $s_C$ is a winning strategy for \charlie in the reachability game over the arena $A_G^a$. Hence, any word over $\Sigma^a$ obtained by playing with $s_C$ is accepted by $A_G^a$, and therefore, by the soundness of the arena construction (\cref{thm:arena-soundness}), any corresponding play $\rounds$ is successful in $G$. Now, recall that $\sigma_C(\evseq)=s_C([\evseq])$ for any event sequence $\evseq$. Hence, $\rounds=\rounds(\sigma_C,\sigma_E)$ for some strategy $\sigma_E$ of \eve. As a result, we can conclude that $\sigma_C$ is a winning strategy for \charlie in $G$.
\end{proof}

Finally, we build a Controller that implements the winning strategy $\sigma_C$, 
provided it exists. First, by \cref{thm:winning-region-soundness-completeness}, 
the existence of $\sigma_C$ implies that $q^a_0 \in W_C$. Next, we define the 
following Moore machine (\cref{def:moore-machine}) based on $s_C$:

\begin{defi}[Controller]
    \label{def:controller-implementation}
Given $A_G^a=(Q^a,\Sigma^a, q_0^a, F^a, \tau^a)$, we define a Controller as $\mathcal{M} = (Q, \Sigma, \Gamma, q_0, \delta, \tau)$, where $Q = Q^a_C \cap W_C$ represents the set of states, $q_{0} = q_0^a$ is the initial state, $\Sigma = \moves_E$ is the input alphabet, $\Gamma = \moves_C$ is the output alphabet, $\delta : Q \times \Sigma \rightarrow Q$ is the transition function, and $\tau : Q \rightarrow \Gamma$ is the output function. The transition function $\delta$ and the output function $\tau$ are defined as follows:
    \begin{align*} 
        \delta(q_C, \move_E) &= \tau^a(s_C(q_C), \move_E) \\ \tau(q_C) &= s_C(q_C).
    \end{align*} 
\end{defi}

Note that by construction the states of $\mathcal{M}$ belong to the winning
region $W_C$ of $A_G^a$, and $\delta$ follows the transition function $\tau^a$
of $A^a_G$. Hence, the output of $\mathcal{M}$ after reading a word $\evseq$ is
exactly $\sigma_C(\evseq)=s_C([\evseq])$ and $\mathcal{M}$ implements
$\sigma_C$, which is a winning strategy by
\cref{thm:winning-region-soundness-completeness}.


\section{Conclusions and Future Work}
\label{sec:conclusions}

Our article presents an effective procedure for synthesizing controllers for timeline-based games, whereas previously, only a proof of the \EXPTIME[2]-completeness of the problem of determining the existence of a strategy was available in the literature. We use a novel construction of a \emph{deterministic} automaton of doubly-exponential (thus optimal) size, which is then adapted to serve as the arena for the game. Then, with standard methods, we solve a reachability game on the arena to effectively compute the winning strategy for the game, if it exists.

This work paves the way for future developments. First, the procedure provided in this article can be realistically implemented and tested. It is conceivable, though, that to avoid the state explosion problem due to the doubly-exponential size of the automaton, it will be necessary to apply \emph{symbolic techniques}. Moreover, an implementation would also need a concrete syntax to specify timeline-based games. Existing languages supported by timeline-based systems (\eg NDDL~\cite{CestaO96} or ANML~\cite{SmithFC08}) might be inadequate for this purpose. Next, as in the case of \LTL, the high complexity makes one wonder whether simpler but still expressive fragments can be found. One possibility might be restricting the synchronization rules to only talk about the \emph{past} concerning the rule's trigger. For co-safety properties (\ie properties expressing the fact that something good will eventually happen) expressed in pure-past \LTL, the realizability problem goes down to being \EXPTIME-complete, and by analogy, this might happen to pure-past timeline-based games as well.

\section*{Acknowledgements}
Luca Geatti and Angelo Montanari acknowledge the support from the 2024 Italian
INdAM-GNCS project ``Certificazione, monitoraggio, ed interpretabilità in
sistemi di intelligenza artificiale'', ref. no. CUP E53C23001670001 as well as
that from the Interconnected Nord-Est Innovation Ecosystem (iNEST), which
received funding from the European Union Next-GenerationEU (PIANO NAZIONALE DI
RIPRESA E RESILIENZA (PNRR) -- MISSIONE 4 COMPONENTE 2, INVESTIMENTO 1.5 -- D.D.
1058 23/06/2022, ECS00000043). In addition, Angelo Montanari acknowledges the
support from the MUR PNRR project FAIR - Future AI Research (PE00000013) also
funded by the European Union Next-GenerationEU. This manuscript reflects only
the authors’ views and opinions, neither the European Union n or the European
Commission can be considered responsible for them. Nicola Gigante acknowledges
the support of the PURPLE project, 1st Open Call for Innovators of the AIPlan4EU
H2020 project, a project funded by EU Horizon 2020 research and innovation
programme under GA n.\ 101016442.

\bibliographystyle{alphaurl}
\bibliography{biblio}

\appendix


\section{}

\runfuncmap*

\begin{proof}
  \proofif%
  We proceed by induction on the length of the event sequence $\evseq =
  \seq{\event_1, \ldots, \event_n}$.
  \begin{description}[before={\renewcommand\makelabel[1]{\bfseries ##1.}},
    labelsep=*, leftmargin=*]
  \item[Base case] 
    If $n = 0$, the only well defined function on an empty codomain is the
    function $\gamma_0: \emptyset \to \emptyset$ with an empty domain, which
    vacuously satisfies the definition of matching function and
    \Cref{lemma:function-matching-run:entire-atom,lemma:function-matching-run:partial-atom}.
    Then, the only run of $\M_\E = (V, D, \emptyset, 0)$ on an empty event
    sequence $\evseq$ is the empty run $\matchseq$ yielding $\M_\E$ itself,
    which vacuously satisfies the definition of run.
  \item[Inductive step]
    Let $\gamma:M_{n} \to [1,\ldots, n]$ be a matching function satisfying
    \Cref{lemma:function-matching-run:entire-atom,lemma:function-matching-run:partial-atom},
    and let $\restrict{\gamma}^{<n}:M_{n-1}\to[1,\ldots,n-1]$ be the restriction
    of $\gamma$ on the domain $M_{n-1}$ defined as the inverse image of
    $[1,\ldots,n-1]$ under $\gamma$, \ie, $M_{n-1} =
    \gamma^{-1}([1,\ldots,n-1])$. $\restrict{\gamma}^{<n}$ is a matching
    function for the event sequence $\slice\evseq_{1,n-1}$ and satisfies
    \cref{lemma:function-matching-run:entire-atom,lemma:function-matching-run:partial-atom}.
    By the inductive hypothesis, there exists a run $\seq{I_1, \ldots, I_{n-1}}$
    of $\M_\E$ on $\slice\evseq_{1,n-1}$, yielding a matching structure
    $\M_{n-1} = (V, D_{n-1}, M_{n-1}, t_{n-1})$. Let $I_n = \gamma^{-1}(n)$, and
    note that $I_n \subseteq \overline{M_{n-1}}$. We show that $\event_n = (A_n,
    \delta_n)$ is an \emph{$I_n$-match} event for $\M_{n-1}$ by breaking the
    proof in steps.

    \statement{$\event_n$ is an \emph{admissible} event for $\M_{n-1}$} Let $T
    \in M_{n-1}$ and $T' \not\in M_{n-1}$. If $D_{n-1}[T',T] = +\infty$,
    $\delta_n \le D_{n-1}[T',T]$ trivially holds. Otherwise, there exists an
    atom $T \before_{l,u} T'$ in $\clause$ and $D_{n-1}[T',T] = u -
    \delta(\slice\evseq_{\restrict\gamma^{<n}(T),n-1})$. We consider two cases
    based on whether $T'$ belongs to the domain of $\gamma$, or not. In the
    first case, $\gamma(T') = n$ and $\delta(\slice\evseq_{\gamma(T),n-1}) +
    \delta_n = \delta(\slice\evseq_{\gamma(T),\gamma(T')}) \le u$, by
    \Cref{lemma:function-matching-run:entire-atom}. In the second case,
    $\delta(\slice\evseq_{\gamma(T),n-1}) + \delta_n =
    \delta(\slice\evseq_{\gamma(T),n}) \le u$, by
    \Cref{lemma:function-matching-run:partial-atom}. In either case, $\delta_n
    \leq u - \delta(\slice\evseq_{\gamma(T),n-1}) = D_{n-1}[T',T]$.

    \statement{\Cref{def:match-event:good-match:start} of
      \Cref{def:match-event}} Let $a[x = v]$ be a quantified token of $\E$. If
    $\tokstart(a) \in I_n$, then $\gamma(\tokstart(a)) = n$ and by definition of
    matching function $\tokstart(x,v)\in A_n$.

    \statement{\Cref{def:match-event:good-match:end} of
      \Cref{def:match-event}}\proofif Let $\tokend(a) \not\in M_{n-1}$ be a
    possible candidate for inclusion in $I_n$. If $\tokstart(a) \in M_{n-1}$ and
    $\tokend(x, v) \in A_n$, then $\tokend(x, v)$ ends the token started at
    $\event_{\restrict\gamma^{<n}(\tokstart(a))}$; otherwise, there would exist
    $\event_i = \pair{A_i, \delta_i}$ prior to $\event_n$ such that $\tokend(x,
    v) \in A_i$, contradicting that $\restrict\gamma^{<n}$ is undefined on
    $\tokend(a)$. By definition of matching function, since $\tokend(x, v) \in
    A_n$ ends the token started at $\event_{\gamma(\tokstart(a))}$, we have
    $\gamma(\tokend(a)) = n$ and $\tokend(a) \in I_n$.

    \statement{\Cref{def:match-event:good-match:end} of
      \Cref{def:match-event}}\proofonlyif If $\tokend(a) \in I_n$, then by
    definition of matching function $\tokend(x,v) \in A_n$. Furthermore, since
    $\tokend(a) \in M_n$, \Cref{lemma:function-matching-run:entire-atom} gives
    $\gamma(\tokstart(a)) \le \gamma(\tokend(a))$ for the atom $\tokstart(a)
    \before_{l,u} \tokend(a)$ in $\clause$. By definition of event sequence,
    $\tokstart(x,v)$ and $\tokend(x,v)$ cannot appear in the same event; hence,
    $\gamma(\tokstart(a)) < \gamma(\tokend(a)) = n$ and $\tokstart(a) \in
    M_{n-1}$.

    \statement{\Cref{def:match-event:preceding-terms} of \Cref{def:match-event}}
    Let $T$ be a term in $I_n$, and let $T'\in V$ be any other term such that
    $D_{n-1}[T',T] \leq 0$. Then, $D_{n-1}[T',T]$ can either be the lower bound
    of an atom $T' \before_{l,u} T$, or the upper bound of an atom $T
    \before_{l,u} T'$ in $\clause$. In the first case, we can directly conclude
    that $T' \in M_{n-1} \cup I_n$, because $T' \in M_n$ by
    \Cref{lemma:function-matching-run:entire-atom} of $\gamma$ and $M_n =
    M_{n-1} \cup I_n$ by definition of $M_{n-1}$ and $I_n$. In the second case,
    note that $D_{n-1}[T',T] = u$, \ie, it has never been decremented because $T
    \not\in M_{n-1}$, and that upper bounds $u$ can never be negative. Thus, $u$
    is equal to $0$ and $\gamma$ satisfies $0 \le
    \delta(\slice\evseq_{\gamma(T),\gamma(T')}) \le 0$
    (\Cref{lemma:function-matching-run:entire-atom}), meaning that $\gamma(T') =
    \gamma(T)$ and $T' \in I_n$.

    \statement{\Cref{def:match-event:lower-bounds} of \Cref{def:match-event}}
    Let $T \in I_n$ and $T' \in M_{n-1}$. $D_{n-1}[T',T]$ cannot be the upper
    bound of an atom $T \before_{l,u} T'$; otherwise,
    \Cref{lemma:function-matching-run:entire-atom} would imply $T \in M_{n-1}$,
    contradicting $T \in I_n$. Thus, $D_{n-1}[T',T]$ must either represent the
    lower bound of an atom $T' \before_{l,u} T$ in $\clause$, or be equal to
    $+\infty$. In the latter case, $\delta_n \ge -D_{n-1}[T',T]$ trivially
    holds. In the former case, $D_{n-1}[T',T] = -l +
    \delta(\slice\evseq_{\restrict\gamma^{<n}(T'),n-1})$. Since $\gamma(T) = n$,
    we have $\delta(\slice\evseq_{\gamma(T'), \gamma(T)}) =
    \delta(\slice\evseq_{\gamma(T'), n}) =
    \delta(\slice\evseq_{\gamma(T'),n-1})+\delta_n \ge l$. Hence, $\delta_n \ge
    l - \delta(\slice\evseq_{\gamma(T'),n-1}) = -D_{n-1}[T',T]$.

    \statement{\Cref{def:match-event:zero-no-bounds} of \Cref{def:match-event}}
    Let $T, T' \in I_n$ be two distinct terms. Then, $\gamma(T') = \gamma(T)$
    and $\delta(\slice\evseq_{\gamma(T'),\gamma(T)}) = 0 $. If $T \before_{l,u}
    T'$ (resp., $T' \before_{l,u} T$) belongs to $\clause$, then $D_{n-1}[T,T']$
    (resp., $D_{n-1}[T',T]$) is the lower bound $l$ and equals 0 by
    \Cref{lemma:function-matching-run:entire-atom}. Otherwise, $D_{n-1}[T,T'] =
    D_{n-1}[T',T] = +\infty$.

    Hence, $\M_{n-1} \stepm[n] \M_n$ is well defined and $\seq{I_1, \ldots,
      I_n}$ is a run of $\M_\E$ on $\evseq$ yielding $\M_n$.
  \end{description}

  \proofonlyif We proceed by induction on the length of the event sequence $\evseq =
  \seq{\event_1, \ldots, \event_n}$.

  \begin{description}[before={\renewcommand\makelabel[1]{\bfseries ##1.}},
    labelsep=*, leftmargin=*]
  \item[Base case] An empty run $\matchseq$ yields $\M_\E = (V, D, \emptyset,
    0)$ itself. Then the function $\gamma_0: \emptyset \to \emptyset$ vacuously
    satisfies the definition of matching function and
    \Cref{lemma:function-matching-run:entire-atom,lemma:function-matching-run:partial-atom}.
  \item[Inductive step]
    Let $\matchseq = \seq{I_1, \ldots, I_n}$ be a run of $\M_\E$ on $\evseq$,
    yielding a matching structure $\M_n = (V, D_n, M_n, t_n)$. Note that
    $\slice\matchseq_{1,n-1}$ is a run of $\M_\E$ on $\slice\evseq_{1,n-1}$
    yielding a matching structure $\M_{n-1} = (V, D_{n-1}, M_{n-1}, t_{n-1})$.
    By the inductive hypothesis, there exists a matching function $\gamma_{<n}:
    M_{n-1} \to [1, \ldots, n-1]$ satisfying
    \Cref{lemma:function-matching-run:entire-atom,lemma:function-matching-run:partial-atom}.
    Let $\gamma: M_{n} \to [1,\ldots, n]$ be the extension of $\gamma_{<n}$ to
    $M_n$, such that $\gamma(T) = n$, for all $T\in I_n$.

    \statement{$\gamma$ is a matching function}
    \Cref{def:matching-function:nodes,def:matching-function:tokens} hold for all
    the terms already present in the domain of $\gamma_{<n}$. For every term in
    $I_n$, \Cref{def:matching-function:nodes} for $\gamma$ follows from
    \Cref{def:match-event:good-match} of $I_n$-match event. Let $\tokstart(a),
    \tokend(a) \in M_n$ be two terms not both already present in $M_{n-1}$,
    meaning that $\tokstart(a) \in M_{n-1}$ and $\tokend(a) \in I_n$, for some
    quantified token $a[x=v]$ in $\E$. By definition of $I_n$-match event,
    $\event_n = (A_n, \delta_n)$ is such that $\tokend(x,v) \in A_n$ and no
    other event in $\slice\event_{\gamma_{<n}(T),n-1}$ contains an action
    $\tokend(x, v)$, otherwise $\tokend(a)$ would already belong to $M_{n-1}$
    (by \Cref{def:match-event:good-match:end} of $I$-match event).
    Thus, $\tokend(x, v) \in A_n$ ends the token started at
    $\event_{\gamma(\tokstart(a))}$, and $\gamma(\tokstart(a))$ and
    $\gamma(\tokend(a))$ correctly identify the endpoints of such token.

    \statement{\Cref{lemma:function-matching-run:entire-atom} of
      \Cref{lemma:function-matching-run}} Let $T \before_{l, u} T'$ be an atom
    in $\clause$, and note that $\gamma$ already satisfies
    \Cref{lemma:function-matching-run:entire-atom} for every $T' \in M_{n-1}$.
    If $T'\in I_n$ instead, consider the entry $D_{n-1}[T,T']$ representing the
    lower bound $l$ of the aforementioned atom. If $D_{n-1}[T,T'] \le 0$,
    \Cref{def:match-event:preceding-terms} of $I$-match event gives $T \in
    M_{n-1} \cup I_n = M_n$. If $D_{n-1}[T,T'] > 0$, $D_{n-1}[T,T']$ no longer
    stores its initial value $-l \leq 0$, meaning that $T$ must have been
    previously matched and $T \in M_{n-1} \subseteq M_n$. In either case, $T\in
    M_n$ and $\gamma(T) \le \gamma(T')$, because $\gamma(T) \le n$.
    If $T \in I_n$, then $\delta(\slice\evseq_{\gamma(T),\gamma(T')}) = 0 \le
    u$, is trivially satisfied by any upper bound $u$. Furthermore, by
    \Cref{def:match-event:zero-no-bounds} of $I$-match event, either the lower
    bound $D_{n-1}[T,T'] = 0$ or the upper bound $D_{n-1}[T',T] = 0$, and they
    both equal their initial values $l$ and $u$. Note that the former case is
    also implied by the latter, so that $l = 0$ and $l \le
    \delta(\slice\evseq_{\gamma(T),\gamma(T')})$. If $T \in M_{n-1}$, by
    \Cref{def:match-event:lower-bounds} of $I$-match event, $\delta_n \ge
    -D[T,T'] = l - \delta(\slice\evseq_{\gamma(T),n-1})$. Hence, $l \le
    \delta(\slice\evseq_{\gamma(T),n-1}) + \delta_n =
    \delta(\slice\evseq_{\gamma(T),\gamma(T')})$. While $\delta_n \le
    D_{n-1}[T',T] = u - \delta(\slice\evseq_{\gamma(T),n-1})$, since $\event_n$
    is an admissible event for $\M_{n-1}$. Hence,
    $\delta(\slice\evseq_{\gamma(T),n-1}) + \delta_n =
    \delta(\slice\evseq_{\gamma(T),\gamma(T')}) \le u$.

    \statement{\Cref{lemma:function-matching-run:partial-atom} of
      \Cref{lemma:function-matching-run}} Let $T \before_{l,u} T'$ be an atom in
    $\clause$ such that $T \in M_n$ and $T' \not\in M_n$. Since $\event_n$ is an
    admissible event for $\M_{n-1}$, $\delta_n \le D_{n-1}[T',T] = u -
    \delta(\slice\evseq_{\gamma(T),n-1})$. Hence,
    $\delta(\slice\evseq_{\gamma(T),n-1}) + \delta_n =
    \delta(\slice\evseq_{\gamma(T),n}) \le u$. \qedhere
  \end{description}
\end{proof}

\superset*
\begin{proof}
  Let $\M_\E \runm*[r]\M_r = (V, D_r, M_r, t_r)$ and $\M_\E \runm*[s] \M_s = (V,
  D_s, M_s, T_s)$, with \linebreak $\gamma_s(\tokstart(a_0)) \le \gamma_r(\tokstart(a_0))$.
  Let $M = M_r \cup M_s$ and $\gamma: M \to [1,\ldots,n]$ be a function defined
  as:
  \[
    \gamma(T) =
    \begin{cases}
      \gamma_s(T) &\text{if } T\in M_s\cap M_r\text{ and }\gamma_s(T)\leq\gamma_r(T)\\
      \gamma_r(T) &\text{if } T \in M_s \cap M_r \text{ and }\gamma_s(T)>\gamma_r(T)\\
      \gamma_s(T) &\text{if } T \in M_s \setminus M_r\\
      \gamma_r(T) &\text{if } T \in M_r \setminus M_s
    \end{cases}
  \]

  \statement{$\gamma$ is a matching function} \Cref{def:matching-function:nodes}
  of \Cref{def:matching-function} for $\gamma$ follows from our hypothesis on $\gamma_s$
  and $\gamma_r$. Regarding \Cref{def:matching-function:tokens}, let
  $\tokstart(a), \tokend(a) \in M$ for some quantified token $a[x=v]$ in $\E$.
  If $\gamma_s$ and $\gamma_r$ map the endpoints of $a$ to the same token in
  $\evseq$, then $\gamma(\tokstart(a))$ and $\gamma(\tokend(a))$ correctly
  identify the endpoints of that token. If instead $\gamma_s$ and $\gamma_r$ map
  $a$ to two distinct tokens in $\evseq$, then $\gamma$ would match $a$
  according to the function whose token comes first, correctly identifying the
  endpoints of such token.

  \statement{$\gamma$ satisfies
    \Cref{lemma:function-matching-run:entire-atom,lemma:function-matching-run:partial-atom}
    of \Cref{lemma:function-matching-run}} Let $T \before_{l,u} T'$ be an atom
  in $\clause$. If $T' \in M$, then either $T' \in M_s$, and $T \in M_s
  \subseteq M$, or $T' \in M_r$, and $T \in M_r \subseteq M$. If $\gamma$ maps
  both terms with either $\gamma_s$ or $\gamma_r$, then $\gamma(T) \leq
  \gamma(T')$ and $l \leq \delta(\slice\evseq_{\gamma(T),\gamma(T')}) \leq u$
  immediately follows. If instead
  $\gamma(T) = \gamma_s(T)$ and $\gamma(T') = \gamma_r(T')$, then $T' \in
    M_r$ and $T \in M_s \cap M_r$. By definition of $\gamma$, $\gamma_s(T) \leq
    \gamma_r(T)$, and, by \Cref{lemma:function-matching-run:entire-atom} for
    $\gamma_r$, $\gamma_r(T) \leq \gamma_r(T')$. Hence, $\gamma(T) \leq
    \gamma(T')$. If $T' \in M_s$, then $\gamma_s(T') > \gamma_r(T')$, and:
    \begin{align*}
      l &\leq\delta_{\gamma_r(T),\gamma_r(T')}&&\text{\Cref{lemma:function-matching-run:entire-atom} for } \gamma_r\\
        &\leq \delta_{\gamma_s(T),\gamma_r(T')}&&\text{}\gamma_s(T)\leq\gamma_r(T)\\
        &< \delta_{\gamma_s(T),\gamma_s(T')} &&\text{} \gamma_s(T') > \gamma_r(T')\\
        &\leq u &&\text{\Cref{lemma:function-matching-run:entire-atom} for } \gamma_s
    \end{align*}
    otherwise:
    \begin{align*}
      l &\leq \delta_{\gamma_r(T),\gamma_r(T')} &&\text{\Cref{lemma:function-matching-run:entire-atom} for } \gamma_r\\
        &< \delta_{\gamma_s(T),\gamma_r(T')}&&\text{} \gamma_s(T) \leq \gamma_r(T)\\
        &\leq \delta_{\gamma_s(T),n} &&\text{} \gamma_r(T') \leq n\\
        &\leq u &&\text{\Cref{lemma:function-matching-run:partial-atom} for } \gamma_s
    \end{align*}
    The case for $\gamma(T) = \gamma_r(T)$ and $\gamma(T') = \gamma_s(T')$ is
    completely symmetrical.

    Lastly, if $T'\not\in M$, but $T \in M$, then either $\gamma(T)=\gamma_s(T)$
    or $\gamma(T) = \gamma_r(T)$, and
    \Cref{lemma:function-matching-run:partial-atom} for $\gamma$ follows from
    \Cref{lemma:function-matching-run:partial-atom} for $\gamma_s$ and
    $\gamma_r$.
\end{proof}

\residualexist*

\begin{proof}
  Let $\gamma(\tokstart(a_0)) = s$, assuming there is no residual matching
  structure $\M_k$ in the sequence $\seq{\M_s, \ldots, \M_{n-1}}$, then for
  every matching structure $\M_i=(V, D_i, M_i, t_i)$, where $s \leq i < n$,
  there exists a pair of terms $\pair{T, T'}$ such that $T \in M_i$ and
  $T'\not\in M_i$, and their distance $D_i[T',T]$ has a finite upper bound. Let
  $E \subseteq V \times V$ be the set that collects all pairs $(T, T')$ for the
  matching structures $\M_i$. We define $\delta_{T, T'}$ as
  $\delta(\slice\evseq_{\gamma(T),\gamma(T')})$ if $\gamma(T')$ is defined, or
  as $\delta(\slice\evseq_{\gamma(T),n})$ otherwise. Let $\delta_E =
  \sum_{\pair{T,T'} \in E} \delta_{T,T'}$ and note that $\delta_E \ge
  \delta(\slice\evseq_{\gamma(\tokstart(a_0)),n})$, because every position in
  $\slice\evseq_{\gamma(\tokstart(a_0)),n}$ is covered by some distance
  $\delta_{T,T'}$. Moreover, each pair $(T, T') \in E$ corresponds to an atom of
  the form $T \before_{l,u} T'$ in $\clause$. According to
  \Cref{lemma:function-matching-run}, we have $\delta_{T, T'} \leq u$, and
  therefore, $\delta_E \leq \window(P)$. Hence, we have
  $\delta(\slice\evseq_{\gamma(\tokstart(a_0)),n}) \leq \delta_E \leq
  \window(P)$: a contradiction hence proving the existence of a residual
  matching structure $M_k$. 
\end{proof}

\soundnessCompleteness*

\begin{proof}
  \proofonlyif Let $\evseq = \seq{\event_1,\ldots,\event_n}$ be a solution plan
  for $P$, and let $\stateseq = \seq{q_0,\ldots,q_n}$ be the run of $\A_P$ on
  $\evseq$. We first show that the sink state is never reached, and then that
  $q_n$ is a final state.

  Let $\event_s = (A_s, \delta_s)$ be the trigger event of a rule $\Rule \equiv
  a_0[x_0 = v_0] \implies \E_1 \lor \dots \lor \E_m$, \ie, $\tokstart(x_0,v_0)
  \in A_s$. Since $\evseq$ is a solution plan, there exist tokens satisfying an
  existential statement $\E$ of $\Rule$ for the trigger $\event_s$. Hence, by
  \cref{lemma:function-matching-run,obs:matching-functions} there exists a run
  $\M_\E \runm* \M_n$, yielding a \emph{closed} matching structure $\M_n$, such
  that $\gamma(\tokstart(a_0)) = s$.

  Let $\overline{\M} = \seq{\M_\E, \M_1,\ldots,\M_{n}}$ be the sequence of all
  the matching structures involved in such run. Note that, by construction
  (\cref{sec:automata-construction}), the states of $\stateseq$ induce all the
  possible runs for the \emph{initial} matching structures of $P$ that can be
  defined on $\evseq$. In particular, the run $\gamma$ must be one of them.
  However, only a subsequence of the matching structures $\overline{M}$ will
  appear in the states of the run $\stateseq$. Indeed, we can identify three key
  points for the sequence $\overline{M}$: the least position $s$ such that
  $\M_s$ is \emph{active} (corresponding to $\gamma(\tokstart(a_0))$), the least
  position $h$ following $s$ such that $\M_h$ no longer belongs to the component
  $\Upsilon$ of the states in $\slice\stateseq_{h, n}$, either because $\M_h$ is
  \emph{closed} or because $\delta(\slice\evseq_{s,h}) > \window(P)$, and the
  least position $k$ following $h$ such that $\M_k$ no longer belongs to the
  component $\Delta$ of the states in $\slice\stateseq_{k,n}$, either because
  $\M_k$ is \emph{closed} or because it gets discarded in favour of the matching
  structures of a later trigger event.

  Every matching structure in $\slice{\overline{M}}_{1,h-1}$ belongs to the
  component $\Upsilon$ of a corresponding state in $\slice\stateseq_{1,h-1}$, so
  the set $\Upsilon$ of the state $q_{s-1}$ is such that $\M_s \in
  \step_{\event_s}(\Upsilon_\bot)$, satisfying condition
  \ref{dfa:delta:trigger-capture} of \cref{sec:automata-construction} for the
  trigger event $\event_s$. Matching structures $\slice{\overline{M}}_{s+1,h}$
  instead belong to the set $\step_\event(\Upsilon^\Rule_t)$, for the partition
  $\Upsilon^\Rule_t$ tracking the satisfaction of the trigger event $\event_s$
  of every state $\slice\stateseq_{s,h-1}$. Hence, all such states satisfy
  condition \ref{dfa:delta:no-failed-step} of \cref{sec:automata-construction}.

  We now show that no \emph{active} matching structures for the trigger event
  $\event_s$ exists after some state $q_h$, following $q_s$ in $\stateseq$. Note
  that the run $\gamma$ yields a \emph{closed} matching structure, and if it
  does so within $\window(P)$ time units from the event
  $\event_{\tokstart(a_0)}$, we identified such position as the closed matching
  structure $\M_h$. So that the state $q_{h-1}$ is such that $\M_h \in
  \step_{\event_h}(\Upsilon^\Rule_t)$, for the partition $\Upsilon^\Rule_t$
  tracking the trigger event $\event_s$, and
  $\step_{\event_h}(\Upsilon^\Rule_t)$ is discarded from $q_h$.

  If instead $\gamma$ yields a \emph{closed} matching structure after
  $\window(P)$ time units from the event $\event_{\tokstart(a_0)}$, lets
  identify such position as $\M_j$, with $j \le n$. If $\M_{j-1}$ belongs to the
  set $\Delta(\E)$ of the state $q_{j-1}$, then $\M_j \in
  \step_{\event_j}(\Delta(\E))$, so that $\step_{\event_j}(\Delta(\E))$ is
  \emph{closed} and discarded from $q_j$, alongside all the other matching
  structures in $\Delta(\E')$, for every $E' \in \Phi(\E)$, \ie, for every other
  existential statement $\E'$ of $\Rule$ still tracking the trigger $\event_s$.
  If instead $\M_{j-1}$ for the trigger event $\event_s$ does not belong to the
  set $\Delta(\E)$ of the state $q_{j-1}$, by construction
  (\cref{sec:automata-construction}), there exist a state $q_h$ in which the
  matching structures tracking $\event_s$ have been replaced by those of a later
  event, and they no longer appear in $\Delta(\E)$ from $q_h$ onwards.

  Since $\evseq$ is a solution plan, the previous argument holds for all the
  trigger events in $\evseq$ of any rules in $S$. Hence, conditions
  \labelcref{dfa:delta:trigger-capture,dfa:delta:no-failed-step} are always met,
  \ie, the sink state is never reached, and no active matching structures belong
  to $q_n$, making it a final state.

  \proofif Let $\evseq = \seq{\event_1,\ldots,\event_n}$ be an event sequence
  accepted by $\A_P$ and let $\rho =\seq{q_0,\ldots,q_n}$ be its accepting run.
  We have to show that the plan corresponding to $\evseq$ is a solution plan for
  $P$, \ie, for every event triggering a rule $\Rule$ in $S$, at least one of
  the existential statements of $\Rule$ is satisfied by $\evseq$.

  Let $\event_s = (A_s, \delta_s)$ be an event in $\evseq$ triggering a rule
  $\Rule \equiv a_0[x_0 = v_0] \implies \E_1 \lor \dots \lor \E_m$, \ie,
  $\tokstart(x_0,v_0) \in A_s$. Since the sink state is never visited in an
  accepting run, the state $q_s$, reached upon reading the event $\event_s$, is
  such that the partition $\Upsilon^\Rule_0$, tracking the satisfaction of the
  trigger event $\event_s$, is not empty. For the same reason, the partition
  $\Upsilon^\Rule_t$ tracking $\event_s$ in every state following $q_s$ can
  never be empty as a result of the function $\step_\event$. However, since the
  final state $q_n$ does not contain any \emph{active} matching structure, there
  must exists a state $q_h$ in $\stateseq$ whose partition
  $\step_{\event_{h+1}}(\Upsilon^\Rule_t)$ gets discarded from $q_{h+1}$. This
  can happen either because $\step_{\event_{h+1}}(\Upsilon^\Rule_t)$ is a
  \emph{closed} set, or because the matchings structures in
  $\step_{\event_{h+1}}(\Upsilon^\Rule_t)$ get promoted to the component
  $\Delta$. In the first case, we can conclude that there exists a run $\M_\E
  \runm* \M_n$ for the initial matching structure $\M_\E$ of an existential
  statement $\E$ of $\Rule$ such that $\M_n$ is \emph{closed} and
  $\gamma(\tokstart(a_0)) = s$, hence, by
  \cref{obs:matching-functions,lemma:function-matching-run}, the trigger event
  $\event_s$ satisfies $\Rule$.

  In the second case, let $\Psi$ be the set of existential statements having an
  active matching structure in $\step_{\event_{h+1}}(\Upsilon^\Rule_t)$, so that
  we can identify them as the sets $\Delta(\E)$, for $\E \in \Psi$, in the
  states from $q_{h+1}$ onwards. By \cref{lemma:residual-matching-structure},
  every such set contains a \emph{residual} matching structure. Hence, by
  \cref{obs:residual-run}, they can become empty only if, at some state $q_{k}$
  following $q_h$, $\step_{\event_{k+1}}(\Delta(\E))$ contains a \emph{closed}
  matching structure for some existential statement $\E \in \Psi$. Note that
  the run $\stateseq$ is an accepting run, so every non-empty set
  $\Delta(\E)$ must become empty before the end of the run. Hence, $q_k$ is
  guaranteed to exist.

  However, it may be the case that, by the time
  $\step_{\event_{k+1}}(\Delta(\E))$ is \emph{closed}, $\Delta(\E)$ no longer
  contains the matching structures for the trigger event $\event_s$, but those
  for a later trigger event $\event_r$ of $\Rule$. Since the sets
  $\Delta(\E)$ store only the matching structures tracking the most recent
  trigger event older than $\window(P)$. Thus, if
  $\step_{\event_{k+1}}(\Delta(\E))$ contains a \emph{closed} matching structure
  for $\event_s$, we can directly assert the existence of a run for $\M_\E$
  implying the satisfaction of $\Rule$ for the trigger event $\event_s$. If
  instead $\step_{\event_{k+1}}(\Delta(\E))$ contains a \emph{closed} matching
  structure $\M_r$ for a later event $\event_r$, there exists a run $\M_\E
  \runm*[r] \M_r$, such that $\gamma_r(\tokstart(a_0)) = r$. Furthermore, by a
  previous consideration on $q_{h+1}$, there exists a run $\M_\E
  \xlongrightarrow{\slice\evseq_{1,h+1},\gamma_s} \hat\M_{h+1}$, yielding a
  \emph{residual} matching structure $\hat\M_{h+1}$, and, by
  \cref{obs:residual-run}, such run can be extended on the entire event sequence
  $\M_\E \runm*[s] \hat\M_n$, to yield a \emph{residual} matching structure
  $\hat\M_n$. Given $\M_\E \runm*[r] \M_r$ and $\M_\E \runm*[s] \hat\M_n$, with
  $\gamma_s(\tokstart(a_0)) \le \gamma_r(\tokstart(a_0))$, by
  \cref{lemma:matching-structure-superset}, there exists a run $\M_\E \runm*[s]
  \M_n$ yielding a matching structure $\M_n$ matching as many terms as $\M_r$
  and such that $\gamma_s(\tokstart(a_0)) = s$. Hence, $\M$ is a \emph{closed}
  matching structure for the existential statement $\E$, and, by
  \cref{obs:matching-functions,lemma:function-matching-run}, $\Rule$ satisfies
  the trigger event $\event_s$.

  Furthermore, all the value duration functions are satisfied by
  the tokens in $\evseq$, being encoded as synchronisation rules by the
  automaton $S_P$. Meanwhile, the automaton $T_P$ guarantees the fulfilment of
  the value transition functions. Hence, we can conclude that $\evseq$ is a
  solution plan for $P$, because every rule in $S$ is satisfied, as well as the
  value duration and value transition functions of every state variable.
\end{proof}


\end{document}